\documentclass[12pt]{article}

\usepackage{graphicx}
\usepackage{float}

\usepackage[letterpaper,margin=1in]{geometry}
\usepackage{amsmath}
\usepackage{amssymb}
\usepackage{amsthm}

\DeclareMathOperator*{\argmin}{argmin} % no space, limits underneath in displays
\DeclareMathOperator{\argmax}{argmax} % no space, limits on side in displays

\usepackage{pifont}% http://ctan.org/pkg/pifont

\usepackage{algorithm}
\usepackage{algorithmic}
\usepackage{tikz}
\usepackage[round,authoryear]{natbib}
\usepackage{hyperref}
\usepackage{newtxtext,newtxmath}
\renewenvironment{abstract}
  {\par
   \noindent
   \begin{center}
   {\LARGE\bfseries Abstract}
   \end{center}
   \par\vspace{1em}
   \noindent}
  {\par}

\date{}

\makeatletter
\renewcommand{\fnum@figure}{\textbf{Figure \thefigure}}
\renewcommand{\fnum@table}{\textbf{Table \thetable}}
\makeatother

\usepackage{scicite}

\usepackage{url}

\usepackage{tcolorbox}

\newcommand{\E}{\mathbb{E}}
\newcommand{\R}{\mathbb{R}}

\newcommand{\bb}{\mathbf{b}}

\newcommand{\eb}{\mathbf{e}}

\newcommand{\vb}{\mathbf{v}}

\newcommand{\wb}{\mathbf{w}}
\newcommand{\xb}{\mathbf{x}}
\newcommand{\yb}{\mathbf{y}}
\newcommand{\zb}{\mathbf{z}}

\newcommand{\Ab}{\mathbf{A}}
\newcommand{\Bb}{\mathbf{B}}

\newcommand{\Hb}{\mathbf{H}}
\newcommand{\Ib}{\mathbf{I}}

\newcommand{\Kb}{\mathbf{K}}

\newcommand{\Mb}{\mathbf{M}}

\newcommand{\Xb}{\mathbf{X}}
\newcommand{\Yb}{\mathbf{Y}}
\newcommand{\Zb}{\mathbf{Z}}

\newcommand{\alphab}{\boldsymbol{\alpha}}
\newcommand{\betab}{\boldsymbol{\beta}}

\newcommand{\phib}{\boldsymbol{\phi}}

\newcommand{\omegab}{\boldsymbol{\omega}}

\newcommand{\mub}{\boldsymbol{\mu}}

\newcommand{\oneb}{\mathbf{1}}

\newcommand{\Hc}{\mathcal{H}}

\newcommand{\Oc}{\mathcal{O}}

\newtheorem{theorem}{Theorem}

\newtheorem{assumption}{Assumption}

\newtheorem{corollary}{Corollary}
\newtheorem{definition}{Definition}

\newtheorem{lemma}{Lemma}

\newtheorem{proposition}[theorem]{Proposition}
\newtheorem{remark}{Remark}

\def\scititle{
	Scalable and Adaptive Distribution Regression via Random Features
}
\title{\bfseries \boldmath \scititle}

\author{
	Yinsong~Wang$^{1}$,
        Xiao~Liu$^{2}$,
        Yu~Ding$^{2}$\and
    \small$^{1}$School of Mechanical, Industrial and Material Engineering,\\
        \small University of Nevada Reno, Reno  89557, USA.\and
	\small$^{2}$H. Milton Stewart School of Industrial and Systems Engineering,\\
        \small Georgia Institute of Technology, Atlanta 30332, USA.\and
}

\begin{document} 

% Insert the title and author list
\maketitle

\begin{abstract}
Distribution regression seeks to learn mappings from probability distributions to response variables and has emerged as a unifying framework for kernel methods, multi-instance learning, and set-valued learning. Despite its success, kernel-based distribution regression remains computationally expensive and sensitive to kernel choice, motivating scalable and adaptive representations. Building on random feature approximations, we develop a finite-sample analysis of Random Feature Distribution Regression (RaDR) under two-stage sampling. Our perturbation analysis separates the effects of distribution-level sampling, finite within-bag sampling, and random feature approximation. We establish sufficient feature-budget conditions under which RaDR preserves the order of our excess-risk upper bound for exact two-stage kernel distribution regression while maintaining computational efficiency. We further propose Adaptive Random Feature Distribution Regression (A-RaDR), a unified score-guided framework that adaptively allocates random features according to downstream learning objectives. Experiments on synthetic benchmarks and 41 real-world datasets demonstrate that RaDR delivers substantial computational savings while maintaining competitive predictive performance, whereas A-RaDR further improves learning accuracy in distribution regression while retaining the computational speed up. Together, these results establish a practical and theoretically grounded framework for adaptive, scalable distribution
learning
\end{abstract}
\clearpage
\tableofcontents
\clearpage

\section{Introduction}\label{sec:intro}

Machine learning has long been built upon a simple but powerful abstraction: data as vectors \citep{vapnik1999overview, hastie2009elements}. From classical linear models and support vector machines to modern deep neural networks, the dominant paradigm in machine learning has traditionally assumed that each observation can be represented as a fixed-dimensional point in a Euclidean space. This vector-centric viewpoint has driven decades of progress in computer vision, speech recognition, scientific computing, and predictive modeling. Even as machine learning architectures have grown increasingly sophisticated, much of modern learning theory and algorithm design remains fundamentally rooted in vector-valued representations \citep{lecun2015deep, bengio2017deep}. 

Yet many real-world learning problems are fundamentally not vector-valued. In numerous scientific and engineering domains, each observation is more naturally represented as a collection, set, bag, or probability distribution rather than a single feature vector \citep{dietterich1997solving, maron1997framework}. A medical diagnosis may rely on an entire collection of pathology images or patient measurements \citep{campanella2019clinical, ilse2018attention}; a manufacturing process may be characterized by particle size distributions evolving over time \citep{van2020predicting}; a sensor network may generate streams of stochastic observations \citep{dennis2018multiple}; and modern language models themselves process increasingly long contexts that are intrinsically set- or sequence-valued. In such settings, the statistical object of interest is no longer an individual sample, but the underlying distribution generating those samples. This shift has fueled growing interest in learning paradigms capable of operating directly on distributions, where the input to the learning system is itself a probability measure estimated through finite observations \citep{smola2007hilbert, foulds2010review, muandet2017kernel}.

Existing approaches to learning from bags, sets, and distributions can be broadly organized into two major paradigms. The first is the multi-instance learning (MIL) paradigm, which treats each input as a finite unordered bag of instances \citep{dietterich1997solving, maron1997framework, foulds2010review}, is increasingly dominated by deep neural architectures \citep{ilse2018attention, zaheer2017deep, lee2019set}. The second is distribution regression, which explicitly formulates the problem as supervised learning from probability distributions to responses \citep{poczos2013distribution, szabo2016learning, muandet2017kernel}. MIL is algorithmically flexible but often statistically opaque and data hungry, whereas distribution regression is principled, stable, and theoretically well understood, but limited by computation.

Distribution regression formulates the set-valued learning problems as supervised learning problems in which each input is an underlying probability distribution observed through a finite bag of samples. Early nonparametric formulations treated the problem directly as regression on probability measures and established consistency under distribution-free assumptions, often through density-based or kernel-kernel constructions \citep{poczos2013distribution}. A more influential line of work later represented distributions through kernel mean embeddings (KME), which map probability measures into a reproducing kernel Hilbert space (RKHS) \citep{smola2007hilbert, sriperumbudur2010hilbert, sriperumbudur2011universality, muandet2017kernel, meunier2022distribution}. This embedding viewpoint provides a natural way to compare distributions through inner products, distances, and kernels on embedded measures, and has become the dominant foundation for kernel-based distribution regression. Within this framework, support measure machines introduced a discriminative learning formulation over probability measures \citep{muandet2012learning}, while subsequent two-stage learning theory established consistency and convergence guarantees when both the distributions and the samples from each distribution are random \citep{szabo2015two, szabo2016learning, bachoc2025improved}. These developments make kernel mean embedding-based distribution regression especially attractive in scientific and safety-critical settings.

Despite these theoretical guarantee, computation remains a major obstacle in large scale applications. Kernel mean embedding-based distribution regression typically requires constructing and manipulating large kernel matrices over both instances and bags, so the cost grows rapidly with the number of distributions and the number of samples per distribution. As a result, the very framework that is arguably the most principled default for distributional supervised learning becomes difficult to deploy in large-scale settings. Moreover, as in conventional kernel learning, performance remains tightly coupled to the choice of kernel \citep{gonen2011multiple, cortes2012algorithms, cortes2010two}. A poorly chosen kernel may yield a representation that is statistically ineffective, while adaptive kernel selection or learning procedures can introduce additional optimization overhead.

Previous studies on kernel approximations for distribution regression have primarily emphasized computational efficiency and practical implementation \citep{flaxman2015supported, mitrovic2016dr, wang2020conformal, adsuara2019nonlinear}, while recent work on adaptive distribution kernels has focused on learning representations suited to downstream tasks \citep{kachaiev2024learning}. These developments have largely pursued efficiency and adaptivity along separate lines, the integration of both scalability and adaptivity into a common framework receives much less attention. Building on these efforts, we develop a finite-sample analysis of random feature distribution regression to theoretically justify its computation and performance trade-off and extend the framework through score-guided feature selection to address scalability and adaptivity together. The paper is organized around three main contributions.

\begin{enumerate}
    \item In Section~\ref{sec:RaDR}, we develop a finite-sample analysis of \emph{Random Feature Distribution Regression} (RaDR) in the classical two-stage sampling setting. Our perturbation analysis separates the effects of distribution-level sampling, finite within-bag sampling, and random feature approximation. For kernel ridge regression with a linear distribution kernel, we establish sufficient feature-budget conditions under which RaDR preserves the order of the excess-risk upper bound of exact two-stage kernel distribution regression, while reducing the computational cost.

    \item In Section~\ref{sec:A-RaDR}, we propose \emph{Adaptive Random Feature Distribution Regression} (A-RaDR), a unified score-guided framework for achieving adaptive kernels in distribution regression. Our formulation encompasses several important adaptive sampling strategies, including kernel alignment, leverage scores, and distribution variance, within a single framework.

    \item In Section~\ref{sec:experiments}, we conduct extensive experiments on synthetic benchmarks and 41 real-world datasets. The results demonstrate that RaDR substantially improves computational efficiency while maintaining competitive predictive performance, and that A-RaDR further improves prediction accuracy in the low-feature regime through adaptive feature selection.
\end{enumerate}

\section{Problem Setup and State-of-the-Art}\label{sec:setup}

In this section, we review the mathematical foundations of kernel mean embeddings and two-stage kernel distribution regression. After introducing their definitions and formulations, we highlight the computational bottlenecks that motivate the scalable approach introduced in the following sections.

\subsection{Notation used throughout the paper}
Before proceeding, we summarize the main notation used in the remainder of the paper. Scalars or functions are denoted by non-bold letters such as $x$, $y$, and $k$. Vectors are denoted by bold lower-case letters such as $\xb$ and $\yb$, and matrices by bold upper-case letters such as $\Xb$ and $\Kb$. Calligraphic letters such as $\mathcal{X}$, $\mathcal{Y}$, and $\mathcal{H}_k$ denote spaces, sets, or function classes. Probability measures on a space $\mathcal{X}$ are written as $\mathbb{P},\mathbb{Q}$, while empirical bags are written as $\mathcal{X}_i = \{\xb_{ij}\}_{j=1}^{N_i}$. We denote inner product with $\langle \cdot, \cdot \rangle$, norm with $\|\cdot\|$, expectation with $\E[\cdot]$, with appropriate subscripts specifying their associated space or measure.

We use $n$ for the number of bags or distributions in the training set. The quantity $N_i$ denotes the size of the $i$th training bag, and $N_t$ denotes the test bag size. The embedding kernel is written as $k$, with RKHS $\mathcal{H}_k$, and the distribution kernel is written as $K$. The population embedding of $\mathbb{P}$ is $\mu_{\mathbb{P}}$, while its empirical embedding from bag $\mathcal{X}_i$ is $\hat\mu_i$.

\subsection{Distribution regression as learning over probability measures}

Distribution regression considers a supervised learning problem with inputs drawn from a collection of probability measures on an instance space $\mathcal{X}$ and outputs taking values in $\mathcal{Y} \subseteq \mathbb{R}$ for regression, or in a finite label space for classification. Let
\[
(\mathbb{P}_i, y_i) \in \mathcal{M}_1^+(\mathcal{X}) \times \mathcal{Y}, 
\qquad i = 1,\dots,n,
\]
denote the population-level training pairs, where $\mathcal{M}_1^+(\mathcal{X})$ is a suitable class of Borel probability measures on $\mathcal{X}$. The target is to learn a function
\[
f^\star : \mathcal{M}_1^+(\mathcal{X}) \to \mathcal{Y}
\]
from the observed training data and then predict the response associated with a new input distribution $\mathbb{P}$.

In practice, however, the distributions $\mathbb{P}_i$ are not directly observed. Instead, for each $i$, one only observes a bag of samples
\[
\mathcal{X}_i = \{\xb_{i1},\dots,\xb_{iN_i}\}, \qquad \xb_{ij} \overset{\text{i.i.d.}}{\sim} \mathbb{P}_i,
\]
where $N_i$ is the bag size for the $i$th distribution. The actual training data therefore take the form
\[
\mathcal{D}_n = \{(\mathcal{X}_i, y_i)\}_{i=1}^n.
\]
As mentioned in the beginning of this Section, we reserve $n$ for the number of distributions, and $N_i$ for the number of instances in bag $i$, since the distinction between these two sampling levels becomes important both computationally and statistically. In particular, later theoretical results will separately track the effect of finite bag sampling and finite random feature approximation.

Distribution regression formulation can therefore be stated as follows: given finitely many bags $\mathcal{X}_1,\dots,\mathcal{X}_n$ and corresponding responses $y_1,\dots,y_n$, estimate a predictor $\hat f$ such that $\hat f(\mathcal{X})$ approximates the response associated with the underlying distribution that generated the bag $\mathcal{X}$. In a nutshell, the covariates are latent probability measures, and the observed bags are noisy finite-sample surrogates for those measures.

\subsection{Kernel mean embeddings of distributions}

A standard nonparametric way to represent probability distributions is through kernel mean embeddings \citep{smola2007hilbert,muandet2017kernel}. Let $k:\mathcal{X}\times\mathcal{X}\to\mathbb{R}$ be a positive definite kernel with associated RKHS $\mathcal{H}_k$. For a distribution $\mathbb{P}\in\mathcal{M}_1^+(\mathcal{X})$ satisfying the usual integrability condition, its kernel mean embedding is defined as
\[
\mu_{\mathbb{P}} = \int_{\mathcal{X}} k(\xb,\cdot)\, d\mathbb{P}(\xb) \in \mathcal{H}_k.
\]
The embedding $\mu_{\mathbb{P}}$ is the RKHS element that averages the feature map $k(\xb,\cdot)$ under $\mathbb{P}$. Equivalently, for any $h \in \mathcal{H}_k$,
\[
\langle h, \mu_{\mathbb{P}} \rangle_{\mathcal{H}_k}
=
\int_{\mathcal{X}} h(\xb)\, d\mathbb{P}(\xb),
\]
where $\langle \cdot, \cdot \rangle_{\mathcal{H}_k}$ denotes the inner product in RKHS $\mathcal{H}_k$, so the mean embedding acts as a linear functional encoding expectations of RKHS test functions. When $k$ is characteristic \citep{sriperumbudur2011universality, simon2018kernel}, the map $\mathbb{P} \mapsto \mu_{\mathbb{P}}$ is injective, meaning that the embedding uniquely determines the distribution. In another word, kernel mean embedding is lossless. This lossless embedding property is one of the key reasons behind the popularity of KME in distribution regression.

Since the population distributions are often not directly available in practice, one replaces $\mu_{\mathbb{P}_i}$ by its empirical counterpart
\[
\hat\mu_i
=
\frac{1}{N_i}\sum_{j=1}^{N_i} k(\xb_{ij},\cdot).
\]
This is simply the average feature map over the bag $\mathcal{X}_i$. The empirical embedding is unbiased for the population embedding and converges in RKHS norm under standard assumptions. Accordingly, the first stage of kernel distribution regression consists of replacing each latent distribution $\mathbb{P}_i$ by the empirical element $\hat\mu_i \in \mathcal{H}_k$.

\subsection{Kernels on embedded distributions}

It is worth noting that in the distribution regression pipeline, there are usually two kernels involved. In the previous subsection, we have introduced the embedding kernel $k$ acting on instances $\xb \in \mathcal{X}$ to represent distributions in $\mathcal{H}_k$. A second kernel, defined on these embedded objects, is then used for the downstream supervised task, we term the downstream supervised learning kernel as distribution kernel to differentiate it from embedding kernel $k$. Let
$
K:\mathcal{H}_k \times \mathcal{H}_k \to \mathbb{R}
$
be a positive definite kernel on the embedding space. Then the induced distribution kernel is
$
K(\mathbb{P},\mathbb{Q})
=
K(\mu_{\mathbb{P}}, \mu_{\mathbb{Q}}).
$
In practice one uses the empirical approximation
$
\hat K_{ij}
=
K(\hat\mu_i,\hat\mu_j).
$
A common and important example is the Gaussian kernel on the RKHS:
\[
K(\mu,\nu)
=
\exp\!\left(
-\frac{\|\mu-\nu\|_{\mathcal{H}_k}^2}{2\tau^2}
\right),
\]
where $\tau>0$ is a bandwidth parameter. Since
\[
\|\mu_{\mathbb{P}}-\mu_{\mathbb{Q}}\|_{\mathcal{H}_k}^2
=
\mathbb{E}_{\xb,\xb'\sim \mathbb{P}} k(\xb,\xb')
+
\mathbb{E}_{\zb,\zb'\sim \mathbb{Q}} k(\zb,\zb')
-
2\mathbb{E}_{\xb\sim\mathbb{P},\,\zb\sim\mathbb{Q}} k(\xb,\zb),
\]
the induced kernel depends on pairwise embedding kernel evaluations across and within bags. Its empirical version is obtained by replacing expectations with bag averages:
\begin{align*}
\|\hat\mu_i-\hat\mu_j\|_{\mathcal{H}_k}^2
&=
\frac{1}{N_i^2}\sum_{a=1}^{N_i}\sum_{b=1}^{N_i} k(\xb_{ia},\xb_{ib})
+
\frac{1}{N_j^2}\sum_{a=1}^{N_j}\sum_{b=1}^{N_j} k(\xb_{ja},\xb_{jb}) \\
&\quad
-
\frac{2}{N_iN_j}\sum_{a=1}^{N_i}\sum_{b=1}^{N_j} k(\xb_{ia},\xb_{jb}).
\end{align*}
This expression already reveals the main computational difficulty, that even a single distribution kernel evaluation requires summing over all cross-bag instance pairs.

More generally, the second-stage kernel $K$ may be chosen according to the downstream task. Linear kernels, radial kernels, and others are all possible. In this paper, we will focus on the linear kernel
$
    K(\mu,\nu) = \langle\mu,\nu\rangle_{\Hc_K},
$
for theoretical analysis, and include Gaussian distribution kernel for numerical studies.

% The present paper focuses on the standard KME-based distribution regression setting in which the distribution predictor is learned from such kernels, and later develops a scalable random feature approximation for this pipeline.

\subsection{Two-stage distribution regression}

The preceding construction leads to a natural two-stage learning paradigm \citep{szabo2016learning}. In the first stage, each input distribution is estimated from a finite bag through its empirical embedding, and transformed into an potentially infinite-dimensional RKHS representation via the embedding kernel $k$. In the second stage, a supervised kernel learning method is applied to the resulting embeddings via the distribution kernel $K$. For kernel ridge regression, this yields the predictor
\[
\hat f(\mathbb{P})
=
\sum_{i=1}^n \alpha_i K(\hat\mu_i,\mu_{\mathbb{P}}),
\]
where the coefficient vector $\alphab = (\alpha_1,\dots,\alpha_n)^\top$ solves
\[
(\Kb_N + \lambda n \Ib_n)\alphab = \yb,
\]
with
\[
(\Kb_N)_{ij} = K(\hat\mu_i,\hat\mu_j), 
\qquad
\yb = (y_1,\dots,y_n)^\top,
\]
and $\lambda>0$ denotes the regularization parameter. For prediction from a test bag
\[
\mathcal{X}_{t} = \{\xb_{t1},\dots,\xb_{tN_t}\},
\]
one first forms the empirical embedding
\[
\hat\mu_t = \frac{1}{N_t}\sum_{j=1}^{N_t}k(\xb_{tj},\cdot),
\]
and then evaluates
\[
\hat f(\mathcal{X}_t)
=
\sum_{i=1}^n \alpha_i K(\hat\mu_i,\hat\mu_t).
\]

Although we focus on kernel ridge regression for concreteness, the same representation can be coupled with other second-stage learners, including support vector machines \citep{vapnik1997support} and other kernel methods. 

\subsection{Computational bottlenecks of exact kernel distribution regression}

The exact KME-based formulation is statistically appealing due to its lossless embedding property \citep{sriperumbudur2011universality}, but its computational cost can be even more prohibitive than instance-level kernel learning. To see this, consider first the construction of the distribution-level Gram matrix $\Kb_N$. For each pair $(i,j)$, evaluating $K(\hat\mu_i,\hat\mu_j)$ generally requires computing quantities of the form
\[
\sum_{a=1}^{N_i}\sum_{b=1}^{N_j} k(\xb_{ia},\xb_{jb}),
\]
together with corresponding within-bag sums. If the average bag size is of order $\bar N$, then a single entry of $\Kb_N$ may cost $\Oc(\bar N^2)$ operations, so forming the full $n\times n$ matrix costs roughly $\Oc(n^2\bar N^2)$ in the balanced case. After this, solving the kernel ridge regression system requires $\Oc(n^3)$ time and $\Oc(n^2)$ memory in the standard implementation. Thus, unlike ordinary kernel regression on vector-valued inputs, distribution regression incurs a nested computational burden: a kernel method at the bag level built on top of kernel operations at the instance level.

\section{Random Feature Distribution Regression}
\label{sec:RaDR}

The computational challenge of kernel distribution regression is shared by kernel methods more broadly and has motivated a substantial literature on scalable kernel approximation \citep{hofmann2008kernel}. Existing approaches include Nystr\"om methods, randomized sketching techniques, and random feature approximations \citep{williams2001using, zhang2008improved, li2014large, mahoney2011randomized, drineas2012fast, zhang2019incremental, rahimi2007random}. Among these, random Fourier features are particularly attractive for shift-invariant kernels because they approximate the kernel function itself through an explicit finite-dimensional feature map rather than a low-rank approximation of the kernel matrix. This explicit representation naturally removes the need to store infinite-dimensional embeddings while preserving the favorable statistical properties of kernel learning \citep{rudi2017generalization, sun2018but, avron2017random, li2021towards}.

For distribution regression, random feature approximations replace implicit kernel mean embeddings with explicit finite-dimensional vectors. Each bag is represented by the average of its instances’ random features, and a supervised learner is fitted to the resulting bag representations. This construction has been used in earlier work on distribution regression, including ecological inference, approximate Bayesian computation, conformal prediction, and remote sensing \citep{flaxman2015supported, mitrovic2016dr, wang2020conformal, adsuara2019nonlinear}. We refer to this common construction as Random Feature Distribution Regression (RaDR) and develop its theoretical analysis under two-stage sampling.

We first introduce random feature approximations for the embedding kernel and construct explicit distribution embeddings. We then formulate RaDR for kernel ridge regression, analyze its computational complexity, and establish a perturbation-based excess-risk bound that separates three sources of error: finite distribution-level sampling, finite within-bag sampling, and random feature approximation. Based on this decomposition, we identify sufficient conditions under which RaDR achieves order-wise lossless approximation relative to exact kernel distribution regression. Finally, we extend the same perturbation framework to binary classification with soft-margin support vector machines (SVMs). Proofs are deferred to Appendix~\ref{app:RaDR_proofs}.

% In the previous section, we introduced kernel mean embedding based two-stage distribution regression and highlighted both its strengths and its computational bottleneck. While kernel mean embeddings provide a lossless representation of probability distributions in reproducing kernel Hilbert spaces, the implicit infinite-dimensional feature map makes the empirical embeddings impossible to represent explicitly. As a result, both training and testing require extensive pairwise kernel evaluations across instances and bags, leading to a nested computational structure whose cost scales multiplicatively with the two stages of distribution regression.

\subsection{Random features for distribution embeddings}
\label{subsec:RaDR_features}

Recall from Section~\ref{sec:setup} that exact kernel distribution regression represents a distribution $\mathbb{P}$ by its kernel mean embedding $\mu_{\mathbb{P}}$ and then applies a distribution kernel to these embeddings. We focus on the linear outer kernel
\begin{equation}
\label{eq:linear_distribution_kernel_sec3}
    K(\mathbb{P},\mathbb{Q})
    =
    \langle \mu_{\mathbb{P}},\mu_{\mathbb{Q}}\rangle_{\mathcal{H}_k},
\end{equation}
which is already expressive when the base kernel $k$ is characteristic, and it isolates the role of random feature approximation, since all approximation error enters through the mean embedding itself.

We focus on shift-invariant kernels of the form
\[
k(\xb,\zb)=\psi(\xb-\zb),
\qquad
\xb,\zb\in\mathcal{X}\subset\mathbb{R}^d,
\]
normalized so that $k(\xb,\xb)=1$. By Bochner's theorem \citep{rudin2017fourier}, there exists a probability measure $\Lambda$ on $\mathbb{R}^d$ such that
\begin{equation}
\label{eq:bochner_sec3}
    k(\xb,\zb)
    =
    \int_{\mathbb{R}^d}
    \cos\{\omegab^\top(\xb-\zb)\}
    \,d\Lambda(\omegab) = \E_{\omegab}[\cos\{\omegab^\top(\xb-\zb)\}].
\end{equation}
In practice, the random feature distribution $\Lambda$ admits analytical forms for popular kernel choices, such as normal distribution for Gaussian kernel, and exponential distribution for Cauchy kernel. We refer the audience to \citep{rahimi2007random} for more details. The expression in Eq.~\eqref{eq:bochner_sec3} reveals that a kernel function $k(\xb,\zb)$ can be approximated via sampling random vectors from the probability measure $\Lambda(\omegab)$. Let
\[
\omegab_1,\ldots,\omegab_D
\overset{\mathrm{i.i.d.}}{\sim}
\Lambda,
\qquad
b_1,\ldots,b_D
\overset{\mathrm{i.i.d.}}{\sim}
\mathrm{Unif}[0,2\pi],
\]
independently. Define
\begin{equation}
\label{eq:rf_map_sec3}
    \phib_D(\xb)
    =
    \sqrt{\frac{2}{D}}
    \big(
    \cos(\omegab_1^\top\xb+b_1),
    \ldots,
    \cos(\omegab_D^\top\xb+b_D)
    \big)^\top
    \in\mathbb{R}^{D}.
\end{equation}
We then construct the corresponding random feature kernel as
\begin{equation}
\label{eq:rf_kernel_sec3}
    k_D(\xb,\zb)
    =
    \langle \phib_D(\xb),\phib_D(\zb)\rangle_2
    =
    \frac{2}{D}\sum_{r=1}^D
    \cos(\omegab_r^\top\xb+b_r)
    \cos(\omegab_r^\top\zb+b_r) \approx k(\xb,\zb).
\end{equation}
The approximation holds by noticing the random-bias construction is unbiased through
\[
\mathbb{E}_{b}\!\left[
2\cos(\omegab^\top\xb+b)\cos(\omegab^\top\zb+b)
\right]
=
\cos\{\omegab^\top(\xb-\zb)\},
\]
and hence
\[
\mathbb{E}_{\omegab,b}\big[k_D(\xb,\zb)\big]=k(\xb,\zb).
\]

For a distribution $\mathbb{P}$, we can define its random feature distribution embedding by
\begin{equation}
\label{eq:rf_distribution_embedding_sec3}
    \mub_{\mathbb{P},D}
    =
    \int_{\mathcal{X}}\phib_D(\xb)\,d\mathbb{P}(\xb)
    \in\mathbb{R}^{D}.
\end{equation}
If $\mathbb{P}_i$ is observed only through a bag
\[
\mathcal{X}_i
=
\{\xb_{i1},\ldots,\xb_{iN_i}\},
\qquad
\xb_{ij}\overset{\mathrm{i.i.d.}}{\sim}\mathbb{P}_i,
\]
then the empirical RF embedding is
\begin{equation}
\label{eq:empirical_rf_distribution_embedding_sec3}
    \widehat{\mub}_{i,D}
    =
    \frac{1}{N_i}\sum_{j=1}^{N_i}\phib_D(\xb_{ij})
    \in\mathbb{R}^{D}.
\end{equation}
Accordingly, the RF approximation to the linear distribution kernel is
\begin{equation}
\label{eq:rf_distribution_kernel_sec3}
    K_D(\mathbb{P},\mathbb{Q})
    =
    \langle \mub_{\mathbb{P},D},\mub_{\mathbb{Q},D}\rangle_2,
\end{equation}
and its empirical version is
\begin{equation}
\label{eq:empirical_rf_distribution_kernel_sec3}
    \widehat K_D(\mathcal{X}_i,\mathcal{X}_j)
    =
    \langle \widehat{\mub}_{i,D},\widehat{\mub}_{j,D}\rangle_2.
\end{equation}

The above construction converts the distribution representation from an RKHS-valued object to a finite-dimensional vector. Next, we construct RaDR under kernel ridge regression.

\subsection{RaDR for kernel ridge regression}
\label{subsec:RaDR_krr}

We now formulate RaDR for kernel ridge regression. For simplicity, this subsection assumes a common training bag size $N_i=N$; the implementation and the empirical studies do not pertain to this restriction.

Let
\begin{equation}\label{eq:rf_matrix}
    \widehat{\Mb}_{D}
=
\begin{bmatrix}
\widehat{\mub}_{1,D}^\top \\
\vdots \\
\widehat{\mub}_{n,D}^\top
\end{bmatrix}
\in\mathbb{R}^{n\times D},
\qquad
\yb=(y_1,\ldots,y_n)^\top.
\end{equation}

The RF Gram matrix over distributions is
\begin{equation}
\label{eq:rf_gram_sec3}
    \widehat{\Kb}_{D,N}
    =
    \widehat{\Mb}_{D}\widehat{\Mb}_{D}^{\top}
    \in\mathbb{R}^{n\times n}.
\end{equation}
For a test distribution $\mathbb{P}$, define
\[
\widehat{\vb}_{D,N}(\mathbb{P})
=
\big(
\langle \widehat{\mub}_{1,D},\mub_{\mathbb{P},D}\rangle_2,
\ldots,
\langle \widehat{\mub}_{n,D},\mub_{\mathbb{P},D}\rangle_2
\big)^\top.
\]
The RaDR kernel ridge predictor is
\begin{equation}
\label{eq:RaDR_dual_sec3}
    \widehat f_{\lambda,N,D}(\mathbb{P})
    =
    \widehat{\vb}_{D,N}(\mathbb{P})^\top
    (\widehat{\Kb}_{D,N}+n\lambda\Ib_n)^{-1}\yb.
\end{equation}
If the test distribution is observed only through a test bag
\[
\mathcal{X}_t
=
\{\xb_{t1},\ldots,\xb_{tN_t}\},
\qquad
\xb_{tj}\overset{\mathrm{i.i.d.}}{\sim}\mathbb{P},
\]
then $\mub_{\mathbb{P},D}$ is replaced by
\[
\widehat{\mub}_{t,D}
=
\frac{1}{N_t}\sum_{j=1}^{N_t}\phib_D(\xb_{tj}).
\]

The finite-dimensional representation also gives a primal form. Since
\[
\widehat f_{\lambda,N,D}(\mathbb{P})
=
\langle \widehat{\wb}_{\lambda,N,D},\mub_{\mathbb{P},D}\rangle_2,
\]
where
\begin{equation}
\label{eq:RaDR_primal_weight_sec3}
    \widehat{\wb}_{\lambda,N,D}
    =
    \widehat{\Mb}_{D}^{\top}
    (\widehat{\Mb}_{D}\widehat{\Mb}_{D}^{\top}+n\lambda\Ib_n)^{-1}\yb,
\end{equation}
we may equivalently compute
\begin{equation}
\label{eq:RaDR_primal_sec3}
    \widehat{\wb}_{\lambda,N,D}
    =
    (\widehat{\Mb}_{D}^{\top}\widehat{\Mb}_{D}+n\lambda\Ib_D)^{-1}
    \widehat{\Mb}_{D}^{\top}\yb,
\end{equation}
via the Woodbury inversion lemma (push-through identity). This primal representation is one of the key computational advantages of RaDR, as exact kernel distribution regression must work with an $n\times n$ Gram matrix whose entries require pairwise within-bag kernel computations. RaDR instead constructs a $D$-dimensional mean embedding vector for each bag and then performs ordinary finite-dimensional ridge regression. Next, we provide a brief discussion on the computational advantage of RaDR.

\subsection{Computation and scalability}
\label{subsec:RaDR_complexity}

The computational difference between exact kernel distribution regression and RaDR can be seen directly from the construction of the Gram matrix. In the balanced case $N_i=N$, evaluating one exact linear distribution-kernel entry requires
\[
K_N(\mathcal{X}_i,\mathcal{X}_j)
=
\frac{1}{N^2}
\sum_{a=1}^N\sum_{b=1}^N
k(\xb_{ia},\xb_{jb}),
\]
and therefore costs $\Oc(N^2)$ kernel evaluations. Constructing the full $n\times n$ Gram matrix costs $\Oc(n^2N^2)$ embedding kernel evaluations, followed by an $\Oc(n^3)$ inversion computation induced by Eq.~\eqref{eq:RaDR_primal_weight_sec3}.

By contrast, RaDR first constructs the bag embeddings
\[
\widehat{\mub}_{i,D}
=
\frac{1}{N}\sum_{j=1}^N\phib_D(\xb_{ij}),
\]
which costs $\Oc(nND)$ random feature evaluations. Once these embeddings are computed, forming the RF Gram matrix costs $\Oc(n^2D)$ if done explicitly, while the primal ridge formulation in Eq.~\eqref{eq:RaDR_primal_sec3} costs $\Oc(nD^2+D^3)$ when $D\lesssim n$, and the dual formulation costs $\Oc(n^2D + n^3)$ for $D\gtrsim n$, where $\lesssim,\gtrsim$ denotes less/greater than or similar to in order. 

In a nutshell, in the regime where $D\lesssim n$ and $D\lesssim N^2$, RaDR replaces the nested pairwise computation of exact kernel distribution regression by a linear pass through the instances followed by finite-dimensional supervised learning. This un-nesting property reduces computational complexity from $\Oc(n^2N^2+ n^3)$ to $\Oc(nND + n^2D +D^3)$, which is even more drastic than the speed up of random features for vector kernel ridge regression.

Moreover, RaDR also supports incremental computation. If a new instance $\xb_{i,N+1}$ is added to bag $i$, the empirical RF embedding can be updated as
\[
\widehat{\mub}_{i,D}^{\,N+1}
=
\frac{N}{N+1}\widehat{\mub}_{i,D}^{\,N}
+
\frac{1}{N+1}\phib_D(\xb_{i,N+1}).
\]
Thus the distribution representation can be maintained online without recomputing all pairwise kernel values. This property becomes especially relevant when distribution-valued inputs are generated from streaming systems or large collections of instances.

\subsection{Risk analysis for RaDR}
\label{subsec:RaDR_risk}

We next state the main risk bound for RaDR. The proof is given in Appendix~\ref{app:RaDR_proofs}. Let $\mathcal{R}[f]$ denote the squared-loss population risk and let
\[
\mathcal{E}(f,f_\rho)
=
\mathcal{R}[f]-\mathcal{R}[f_\rho]
\]
be the excess risk, where $f_\rho$ denotes the optimal population predictor. We use $\widehat f_{\lambda}$ for the ideal one-stage KRR estimator that observes the true population embeddings, $\widehat f_{\lambda,N}$ for the exact two-stage estimator that uses finite training bags but no random features, and $\widehat f_{\lambda,N,D}$ for the RaDR estimator. The analysis is built on the following decomposition
\[
\widehat f_{\lambda,N,D}-f_\rho
=
(\widehat f_{\lambda}-f_\rho)
+
(\widehat f_{\lambda,N}-\widehat f_{\lambda})
+
(\widehat f_{\lambda,N,D}-\widehat f_{\lambda,N}).
\]
The first term is the ordinary one-stage KRR statistical error. The second term is the finite-bag perturbation already present in exact kernel distribution regression. The third term is the new RF approximation perturbation introduced by RaDR.

We first state the one-stage exact-kernel bound abstractly. We assume there exists a function 
$
B_{\mathrm{1s}}(n,\lambda,\delta)
$ 
such that, with probability at least $1-\delta$,
\[
\mathcal{E}(\widehat f_\lambda,f_\rho)
\le
B_{\mathrm{1s}}(n,\lambda,\delta),
\]
the assumption is formally stated in the Appendix. The quantity $B_{\mathrm{1s}}(n,\lambda,\delta)$ is an abstraction that represents a standard one-stage excess-risk bound for kernel ridge regression (KRR). Such bounds are well
studied in the classical source-capacity framework for nonparametric regression; see,
for example, \citet{zhang2005learning} and \citet{caponnetto2007optimal}. 
This abstraction allows the RaDR perturbation analysis to be combined with any valid
one-stage KRR theorem under the desired source and capacity conditions. We provide a concrete example for $B_{\mathrm{1s}}(n,\lambda,\delta)$ in the Appendix. With this abstraction, we now state the following theorem.

\begin{theorem}
\label{thm:exact_twostage_main}
Under the boundedness, source, and one-stage baseline conditions stated in Appendix~\ref{app:RaDR_proofs}, there exists a constant $C_2>0$, independent of $n,N,\lambda,\delta$, such that, with probability at least $1-\delta$,
\begin{equation}
\label{eq:exact_twostage_main_sec3}
    \mathcal{E}(\widehat f_{\lambda,N},f_\rho)
    \le
    2B_{\mathrm{1s}}(n,\lambda,\delta/2)
    +
    C_2\left[
    \frac{\log(2n/\delta)}{N\lambda^2}
    +
    \frac{\log(2n/\delta)}{N\lambda^4}
    \right].
\end{equation}
\end{theorem}

Theorem~\ref{thm:exact_twostage_main} provides a rate bound baseline for the exact kernel distribution regression problem. Our bound is a self-contained alternative to \citet{szabo2016learning}'s result, where we separate the one-stage learning and the finite sample effect for clarity. Next, we present the excess-risk bound for RaDR.

\begin{theorem}
\label{thm:RaDR_excess_main}
Under the boundedness, spectral moment, source, and one-stage baseline conditions stated in Appendix~\ref{app:RaDR_proofs}, let $\epsilon_D(\delta)$ denote the uniform random feature approximation radius defined in Appendix~\ref{app:RaDR_proofs}. Then there exists a constant $C_3>0$, independent of $n,N,D,\lambda,\delta$, such that, with probability at least $1-\delta$,
\begin{align}
\label{eq:RaDR_excess_main}
    \mathcal{E}(\widehat f_{\lambda,N,D},f_\rho)
    \le\;&
    4B_{\mathrm{1s}}(n,\lambda,\delta/3) \notag\\
    &+
    C_3\left[
    \frac{\log(3n/\delta)}{N\lambda^2}
    +
    \frac{\log(3n/\delta)}{N\lambda^4}
    +
    \frac{\epsilon_D(\delta/3)^2}{\lambda^2}
    +
    \frac{\epsilon_D(\delta/3)^2}{\lambda^4}
    \right].
\end{align}
\end{theorem}

Theorem~\ref{thm:RaDR_excess_main} folds the random feature approximation effect into the exact kernel distribution regression excess-risk bound in Theorem~\ref{thm:exact_twostage_main}. The first term is the one-stage learning error over true distribution embeddings. The second and third terms are the price of observing each distribution through only $N$ samples. The last two terms are the price of replacing the exact RKHS embedding by $D$ random feature basis. The regularization term $\lambda$ plays an important role in interpreting the results. A smaller $\lambda$ improves approximation bias in kernel ridge regression, but it amplifies both the finite-bag perturbation and the random-feature perturbation.

Based on the above two Theorem statements, we further provide an additional result on test distribution finite sample effect in the Appendix.

\begin{remark}
Theorems~\ref{thm:exact_twostage_main} and~\ref{thm:RaDR_excess_main} are population excess-risk statements that evaluate predictors on true test distributions. In deployment, a fresh test distribution is typically observed only through a test bag of size $N_t$. Appendix~\ref{app:RaDR_deployment} shows that this creates an additional pointwise perturbation of order $N_t^{-1/2}$ with a single factor $\lambda^{-1}$.
\end{remark}

\subsection{Order-wise No-loss regimes}
\label{subsec:RaDR_noloss}

From Theorems~\ref{thm:exact_twostage_main} and \ref{thm:RaDR_excess_main}, one can observe that RaDR is able to achieve the same order wise excess-risk upper bound when the finite sample terms are order wise equal to the random feature approximation terms. Specifically, RaDR is in a no-loss regime whenever
\begin{equation}
\label{eq:RaDR_noloss_condition_general}
    \epsilon_D(\delta/3)^2\bigl(\lambda^{-2}+\lambda^{-4}\bigr)
\lesssim
\frac{\log(3n/\delta)}{N}\bigl(\lambda^{-2}+\lambda^{-4}\bigr).
\end{equation}
where $\asymp$ denotes asymptotic equivalence. Therefore, one can show that a sufficient condition for Eq.~\eqref{eq:RaDR_noloss_condition_general} to hold is
\begin{equation}
\label{eq:RaDR_feature_bag_match}
\epsilon_D(\delta/3)^2
\le
C_0\frac{\log(3n/\delta)}{N}.
\end{equation}

Proposition~2 of \citet{sutherland2015error} gives a
uniform kernel-approximation bound on a compact instance
domain under a finite second spectral moment.
Consequently, for fixed $\delta\in(0,1)$ and fixed
kernel parameters, there exists a constant
$C_\delta>0$ such that
\[
\epsilon_D(\delta/3)^2
\le C_\delta\frac{\log(eD)}{D},
\]
for all sufficiently large $D$. $e$ here is the base of natural logarithm.
The derivation is included in the Appendix. Hence a convenient
sufficient feature-budget condition is
\begin{equation}
\label{eq:RaDR_feature_budget_match}
\frac{\log(eD)}{D}
\lesssim
\frac{\log n}{N},
\end{equation}
which is true when $D \asymp N$. Based on the above observation, we next state the following Corollary.

\begin{corollary}
\label{cor:RaDR_D_threshold_app}
Under the assumptions of Theorem~\ref{thm:RaDR_excess_main},
fix $\delta\in(0,1)$. Suppose that $N\to\infty$ as
$n\to\infty$, $N\le n^p$ for some fixed $p>0$, and
$D\asymp N$. Let $\lambda=\lambda_n>0$ satisfy the chosen
one-stage baseline conditions at confidence level $\delta/3$.
Then RaDR preserves the exact two-stage upper-bound order with probability at
least $1-\delta$ for all sufficiently large $n$.

Moreover, denoting the computational complexity of RaDR with $T_{\mathrm{RaDR}}$, and exact kernel distribution regression with $T_{\mathrm{exact}}$, choosing the less expensive primal or dual formulation gives
\[
T_{\mathrm{RaDR}}=\Oc(nN^2),
\qquad
T_{\mathrm{exact}}
=\Theta(n^2N^2+n^3),
\]
therefore,
\[
T_{\mathrm{RaDR}}/
T_{\mathrm{exact}}
=\Oc(n^{-1}).
\]
\end{corollary}

The proof of Corollary~\ref{cor:RaDR_D_threshold_app} is given in Appendix~\ref{app:RaDR_noloss_proofs}. Corollary~\ref{cor:RaDR_D_threshold_app} provides an important insight that RaDR can achieve a significant speed up while provably preserving the exact order-wise excess-risk upper bound of two-stage kernel distribution regression. At the same number of bags and instances per bag, the training cost decreases at the rate of $\Oc(n^{-1})$. This result is true regardless of whether $D \lesssim n$ is true or not. Moreover, the above Corollary is a conservative characterization of the computational speed up for RaDR. As we will show later on in the empirical results, RaDR training time remains significantly shorter than that of exact kernel distribution regression when $D$ is larger than both $N$ and $n$, while the performance gap vanishes.

\subsection{Extension to binary classification}
\label{subsec:RaDR_svm}

Theoretical properties of kernel distribution regression are predominantly presented in terms of kernel ridge regression, yet one of the most common usages of kernel distribution regression is for classification \citet{kachaiev2024learning}. Therefore, in this subsection, we briefly discuss how the RaDR construction extends to binary classification with soft-margin SVM. Let $y_i\in\{-1,+1\}$ and consider the hinge loss
\[
\ell_{\mathrm{h}}(u)=(1-u)_+.
\]
We investigate RaDR SVM under a linear SVM on the RF distribution embeddings:
\begin{equation}
\label{eq:svm_rf_primal_sec3}
    \mathcal{J}_{n,N,D}(\betab)
    =
    \frac{1}{n}\sum_{i=1}^n
    \ell_{\mathrm{h}}\!\left(y_i\betab^\top\widehat{\mub}_{i,D}\right)
    +
    \frac{\lambda}{2}\|\betab\|_2^2.
\end{equation}
In this case, the same finite-dimensional representation used for RaDR regression also provides a scalable distribution classifier. For the RF approximation, it is useful to compare the exact and RF SVM objectives through their duals. Let
\[
\Yb=\mathrm{diag}(y_1,\ldots,y_n),
\qquad
\mathcal{A}
=
\{\alphab\in\mathbb{R}^n:0\le \alpha_i\le 1/n\}.
\]
The exact two-stage SVM dual is
\begin{equation}
\label{eq:svm_exact_dual_sec3}
    \mathcal{D}_{N}(\alphab)
    =
    \sum_{i=1}^n \alpha_i
    -
    \frac{1}{2\lambda}\alphab^\top \Yb \Kb_N \Yb \alphab,
    \qquad \alphab\in\mathcal{A},
\end{equation}
where $\Kb_N$ is the exact empirical distribution-kernel matrix. The RF dual replaces $\Kb_N$ by $\widehat{\Kb}_{D,N}$:
\begin{equation}
\label{eq:svm_rf_dual_sec3}
    \mathcal{D}_{N,D}(\alphab)
    =
    \sum_{i=1}^n \alpha_i
    -
    \frac{1}{2\lambda}\alphab^\top \Yb \widehat{\Kb}_{D,N} \Yb \alphab.
\end{equation}

\begin{proposition}
\label{prop:svm_dual_sec3}
Suppose the element wise RF approximation error satisfies
\[
\max_{1\le i,j\le n}
\left|
[\widehat{\Kb}_{D,N}]_{ij}-[\Kb_N]_{ij}
\right|
\le
\epsilon_D,
\]
the exact and RF SVM dual objectives satisfy
\begin{equation}
\label{eq:svm_dual_perturb_sec3}
    \sup_{\alphab\in\mathcal{A}}
    \left|
    \mathcal{D}_{N,D}(\alphab)-\mathcal{D}_{N}(\alphab)
    \right|
    \le
    \frac{\epsilon_D}{2\lambda}.
\end{equation}
Consequently, the corresponding optimal dual objective values differ by at most $\epsilon_D/(2\lambda)$.
\end{proposition}

Proposition~\ref{prop:svm_dual_sec3} establishes that, under uniform kernel approximation, the RF SVM dual objective remains close to the exact two-stage SVM dual objective. As for a complete excess-risk theorem mirroring Theorem~\ref{thm:RaDR_excess_main}, we leave it to future extensions.

\section{Enable Adaptivity with Score-based Random Features}
\label{sec:A-RaDR}

In Section 3, we analyzed RaDR under two-stage sampling and characterized the relationship between its feature budget, excess-risk upper bound, and computational cost. Beyond their computational advantages, random features possess another important property that distinguishes them from other kernel approximation methods. Unlike Nystr\"om approximations or randomized sketching methods, whose approximations are largely determined by sampled columns or matrix structure, random features are generated directly from an explicit sampling distribution over the kernel's spectral measure. This explicit representation naturally permits the feature distribution to be modified, reweighted, or selectively sampled while maintaining a finite-dimensional embedding, making random features particularly well suited for adaptive kernel learning \citep{rahimi2007random, rudi2017generalization, sun2018but, avron2017random, li2021towards, yang2012nystrom}.

The importance of adaptivity stems from a fundamental characteristic of kernel methods: their statistical performance is closely tied to the choice of kernel \citep{smola1998learning,hofmann2008kernel}. Consequently, a substantial body of work has studied kernel learning and multiple kernel learning to better align kernel representations with the underlying data geometry or downstream prediction task \citep{cristianini2001kernel, lanckriet2004learning, bach2004multiple, sonnenburg2006large, cortes2010two}. While effective, these approaches often require repeated kernel construction or additional optimization procedures, potentially diminishing the computational benefits offered by scalable kernel approximations.

A particularly relevant direction is score-based random feature sampling. Instead of drawing random features solely from the original spectral measure, these methods modify the feature distribution according to task-specific information such as prediction objectives, kernel leverage, or alignment criteria \citep{sinha2016learning, shahrampour2018data, avron2017random}. This viewpoint bridges kernel approximation and kernel learning: the same finite-dimensional representation that enables scalability can also be adapted to allocate computational resources toward features that are most informative for prediction.

Motivated by this perspective, we propose Adaptive Random Feature Distribution Regression (A-RaDR), a score-guided extension of RaDR that incorporates adaptive feature selection into distribution regression. The framework expresses established criteria---including kernel alignment \citet{sinha2016learning}, leverage scores \citep{li2021towards}, and distribution variance \citep{kachaiev2024learning}---through a common score-matrix formulation applied to distribution embeddings. These scores guide the selection of a feature subset from a candidate pool, allowing the representation to adapt to the chosen criterion within the same finite-dimensional RaDR framework.

\subsection{The two-stage random feature sampling formulation}
\label{subsec:adaptive_dist_features}

We continue with the random-bias cosine features introduced in
Section~\ref{subsec:RaDR_features}. For simplicity, we denote a
frequency--bias pair by $\eta=(\omegab,b)$ and its sampling
distribution by $\pi=\Lambda\otimes\mathrm{Unif}[0,2\pi]$.
To examine the contribution of an individual feature, we further
denote its population and empirical averages by
\begin{align}
    z_{\mathbb P}(\eta)
    &=
    \sqrt{2}\int_{\mathcal X}
    \cos(\omegab^\top\xb+b)\,d\mathbb P(\xb),
    \label{eq:adaptive_scalar_dist_feature}\\
    \widehat z_i(\eta)
    &=
    \frac{\sqrt{2}}{N_i}\sum_{j=1}^{N_i}
    \cos(\omegab^\top\xb_{ij}+b).
    \label{eq:adaptive_empirical_scalar_feature}
\end{align}
Here $\mathbb P\in\mathcal M_1^+(\mathcal X)$, and
$\widehat z_i(\eta)$ estimates $z_{\mathbb P_i}(\eta)$ using the
observed bag $\mathcal X_i$. By the random-bias cosine identity,
the linear distribution kernel satisfies
\begin{equation}
\label{eq:adaptive_dist_kernel_integral}
    K(\mathbb P,\mathbb Q)
    =
    \mathbb E_{\eta\sim\pi}
    \big[z_{\mathbb P}(\eta)z_{\mathbb Q}(\eta)\big].
\end{equation}

To connect this notation with Section~\ref{sec:RaDR}, notice that
the $r$th coordinate of $\mub_{\mathbb P,D}$ is
$D^{-1/2}z_{\mathbb P}(\eta_r)$.
Collecting the empirical averages of one feature across all bags
therefore gives
\begin{equation}
\label{eq:adaptive_z_vector}
    \widehat{\zb}(\eta)
    =
    \big(\widehat z_1(\eta),\ldots,\widehat z_n(\eta)\big)^\top
    \in\mathbb R^n,
    \qquad
    \widehat{\Mb}_D
    =
    \frac{1}{\sqrt D}
    \begin{bmatrix}
        \widehat{\zb}(\eta_1)&\cdots&\widehat{\zb}(\eta_D)
    \end{bmatrix}.
\end{equation}
Thus, each vector $\widehat{\zb}(\eta_r)$ represents one column
of the existing RF embedding matrix, up to the factor
$D^{-1/2}$. This arrangement allows us to evaluate and select
features individually.

The first stage of A-RaDR starts by drawing a candidate pool
$\widetilde\eta_1,\ldots,\widetilde\eta_{M_0}
\overset{\mathrm{i.i.d.}}{\sim}\pi$, independently of the training data, where $M_0\ge D$. For convenience, we use $\widehat{\zb}_m=\widehat{\zb}(\widetilde\eta_m)$ and $[M_0]=\{1,\ldots,M_0\}$. The corresponding pool embedding matrix and Gram matrix are
\begin{equation}\label{eq:adaptive_pool_gram}
    \Zb_0
    =
    \frac{1}{\sqrt{M_0}}
    \begin{bmatrix}
        \widehat{\zb}_1&\cdots&\widehat{\zb}_{M_0}
    \end{bmatrix}
    \in\mathbb R^{n\times M_0}, \qquad
    \widehat{\Kb}_0
    =
    \Zb_0\Zb_0^\top.
\end{equation}

In the second stage, we score the candidates and select a subset $S=\{m_1,\ldots,m_D\}\subseteq[M_0]$. The selected embedding of bag $i$ becomes
\begin{equation}
\label{eq:adaptive_selected_embedding}
    \widehat{\mub}_{i,S}
    =
    \frac{1}{\sqrt D}
    \big(
        \widehat z_i(\widetilde\eta_{m_1}),
        \ldots,
        \widehat z_i(\widetilde\eta_{m_D})
    \big)^\top.
\end{equation}
Accordingly, the selected training matrix is
$\widehat{\Mb}_S
=D^{-1/2}[\widehat{\zb}_{m_1},\ldots,\widehat{\zb}_{m_D}]$,
whose $i$th row is $\widehat{\mub}_{i,S}^\top$.
Its Gram matrix satisfies
\begin{equation}
\label{eq:adaptive_selected_kernel}
    \widehat{\Kb}_S
    =
    \widehat{\Mb}_S\widehat{\Mb}_S^\top
    =
    \frac{1}{D}\sum_{m\in S}
    \widehat{\zb}_m\widehat{\zb}_m^\top.
\end{equation}
The $(i,j)$th entry is therefore the selected RF distribution
kernel
$\widehat K_S(\mathcal X_i,\mathcal X_j)
=\langle\widehat{\mub}_{i,S},\widehat{\mub}_{j,S}\rangle_2$.

The same selected features are used for prediction. For a test
distribution $\mathbb P$, its embedding is
\begin{equation}
\label{eq:adaptive_selected_test_embedding}
    \mub_{\mathbb P,S}
    =
    \frac{1}{\sqrt D}
    \big(
        z_{\mathbb P}(\widetilde\eta_{m_1}),
        \ldots,
        z_{\mathbb P}(\widetilde\eta_{m_D})
    \big)^\top.
\end{equation}
When $\mathbb P$ is observed through a test bag $\mathcal X_t$,
we replace the population averages by their empirical counterparts
in \eqref{eq:adaptive_empirical_scalar_feature}, using $N_t$
samples, and denote the resulting embedding by
$\widehat{\mub}_{t,S}$.

\subsection{A general trace-based score}
\label{subsec:adaptive_general_score}

Adaptive random feature methods select features according to
different objectives, such as response alignment, kernel geometry, or variation across inputs \citep{li2019towards, wang2021orcca}. We next show that these objectives can be expressed through a common trace formulation, which leads to a simple feature selection rule.

Specifically, for a score matrix $\Bb\in\mathbb R^{n\times n}$
and a distribution-level Gram matrix $\Kb$, consider
\begin{equation}
\label{eq:adaptive_trace_objective}
    \mathcal J_B(\Kb)=\mathrm{tr}(\Kb\Bb).
\end{equation}
Since each candidate contributes a rank-one matrix to
\eqref{eq:adaptive_selected_kernel}, its contribution to the
objective is determined by
\begin{equation}
\label{eq:adaptive_general_score}
    q_{B,m}
    =
    \widehat{\zb}_m^\top\Bb\widehat{\zb}_m
    =
    \mathrm{tr}
    \big(\widehat{\zb}_m\widehat{\zb}_m^\top\Bb\big).
\end{equation}
For computation, we form the score vector
\begin{equation}
\label{eq:adaptive_score_vector}
    (q_{B,1},\ldots,q_{B,M_0})^\top
    =
    M_0\,\mathrm{diag}(\Zb_0^\top\Bb\Zb_0),
\end{equation}
where $\mathrm{diag}$ represents the diagonal operator.

The following proposition establishes the selection rule and its comparison with vanilla random features.

\begin{proposition}
\label{prop:adaptive_trace_selection}
Fix the training data and a realized candidate pool. Let $\Bb$ be a score matrix held fixed across all candidate subsets. Then, for every $S\subseteq[M_0]$ with $|S|=D$, 
\begin{equation}
\label{eq:adaptive_selected_trace}
    \mathcal J_B(\widehat{\Kb}_S)
    =
    \frac{1}{D}\sum_{m\in S}q_{B,m}.
\end{equation}
Consequently, a maximizing subset
\begin{equation}
\label{eq:adaptive_selected_set}
    S_B
    \in
    \argmax_{\substack{S\subseteq[M_0]\\|S|=D}}
    \mathcal J_B(\widehat{\Kb}_S)
\end{equation}
is obtained by selecting the $D$ largest scores.

Moreover, let $\widehat{\Kb}_{D,N}$ be the vanilla RaDR
Gram matrix constructed from the first $D$ candidates, we have the following identity,
\[
    \mathcal J_B(\widehat{\Kb}_{S_B})
    \ge
    \mathcal J_B(\widehat{\Kb}_{D,N}).
\]
where $\Kb_{S_B}$ denote the Gram matrix formed with features selected via $\Bb$. If $\Bb$ is also fixed conditional on the training data and
independent of the candidate pool, then
\begin{equation}
\label{eq:adaptive_expected_dominance}
    \mathbb E_{\mathrm{pool}}
    \big[\mathcal J_B(\widehat{\Kb}_{S_B})\big]
    \ge
    \mathbb E_{\mathrm{pool}}
    \big[\mathcal J_B(\widehat{\Kb}_{D,N})\big]
    =
    \mathcal J_B(\Kb_N),
\end{equation}
where $\Kb_N$ is the exact finite-bag Gram matrix.
The expectations are conditional on the training data and taken
over the i.i.d.\ candidate pool drawn from $\pi$.
\end{proposition}

The proof is given in
Appendix~\ref{app:adaptive_trace_selection}.
Proposition~\ref{prop:adaptive_trace_selection} shows that a trace
objective reduces feature selection to ranking scalar scores.
Once $\Bb$ has been constructed, retaining the top-$D$ candidates
maximizes that objective over all subsets of the pool.
This gives a direct way to incorporate different learning
objectives into RaDR through the choice of $\Bb$.

\subsection{Adaptive RaDR algorithm}
\label{subsec:adaptive_algorithm}

With the selected features, A-RaDR proceeds through the same
ridge regression formulation as in
Section~\ref{subsec:RaDR_krr}. In particular, the predictor has
the equivalent dual and primal expressions
\begin{equation}
\label{eq:adaptive_RaDR_estimator}
\begin{aligned}
    \widehat f_{\lambda,N,S_B}(\mathbb P)
    &=
    \mub_{\mathbb P,S_B}^\top
    \widehat{\Mb}_{S_B}^\top
    (\widehat{\Kb}_{S_B}+n\lambda\Ib_n)^{-1}\yb\\
    &=
    \mub_{\mathbb P,S_B}^\top
    (\widehat{\Mb}_{S_B}^\top\widehat{\Mb}_{S_B}
      +n\lambda\Ib_D)^{-1}
    \widehat{\Mb}_{S_B}^\top\yb.
\end{aligned}
\end{equation}
For prediction from $\mathcal X_t$, we use
$\widehat{\mub}_{t,S_B}$ in place of
$\mub_{\mathbb P,S_B}$.
Algorithm~\ref{alg:A-RaDR} summarizes the procedure.

\begin{algorithm}[h]
\caption{Adaptive Random Feature Distribution Regression}
\label{alg:A-RaDR}
\small
\begin{algorithmic}[1]
\REQUIRE Training bags $\{(\mathcal X_i,y_i)\}_{i=1}^n$;
candidate pool size $M_0\ge D$; random features count $D$; ridge parameter $\lambda>0$; score matrix $\Bb$.
% \ENSURE Predictor $\widehat f_{\lambda,N,S_B}$.

\STATE Draw $\widetilde\eta_1,\ldots,\widetilde\eta_{M_0}
\overset{\mathrm{i.i.d.}}{\sim}\pi$.

\STATE Compute the candidate vectors $\widehat{\zb}_m$
using \eqref{eq:adaptive_empirical_scalar_feature}
and \eqref{eq:adaptive_z_vector}.

\STATE Evaluate the scores in \eqref{eq:adaptive_general_score}
and retain the top-$D$ subset $S_B$ (bottom if trace objective is minimization).

\STATE Assemble the selected embedding matrix
$\widehat{\Mb}_{S_B}$.

\STATE Fit the ridge predictor in
\eqref{eq:adaptive_RaDR_estimator} using the primal or dual form.

\RETURN $\widehat f_{\lambda,N,S_B}$.
\end{algorithmic}
\end{algorithm}

The score matrix determines how the feature budget is allocated.
We next describe three choices corresponding to kernel alignment,
leverage scores, and distribution variance.

\subsection{Kernel alignment score}
\label{subsec:adaptive_kernel_alignment}

Kernel alignment aims to construct kernels whose similarity
structure agrees with the responses of the learning task
\citep{cristianini2001kernel,sinha2016learning,
shahrampour2018data}. For distribution regression, this suggests
selecting features whose bag-level averages align with the
observed responses.

For scalar responses, the response similarity matrix is
$\yb\yb^\top$. Accordingly, choosing
$\Bb_{\mathrm{KA}}=\yb\yb^\top$ gives the trace objective
\begin{equation}
\label{eq:adaptive_KA_trace}
    \mathcal J_{\mathrm{KA}}(\Kb)
    =
    \mathrm{tr}(\Kb\Bb_{\mathrm{KA}})
    =
    \yb^\top\Kb\yb.
\end{equation}
Substituting this choice into
\eqref{eq:adaptive_general_score} yields
\begin{equation}
\label{eq:adaptive_KA_score}
    q_{\mathrm{KA},m}
    =
    \widehat{\zb}_m^\top\Bb_{\mathrm{KA}}\widehat{\zb}_m
    =
    (\yb^\top\widehat{\zb}_m)^2.
\end{equation}
Thus, the score favors features with a large absolute inner
product with the response vector across bags. For binary classification, the same formulation applies with
$y_i\in\{-1,+1\}$. For multiclass classification, one may use
$\Bb_{\mathrm{KA}}=\mathbf L\mathbf L^\top$, where
$\mathbf L\in\mathbb R^{n\times C}$ contains one-hot or
contrast-coded labels for $C$ classes. The corresponding score
is then $\|\mathbf L^\top\widehat{\zb}_m\|_2^2$.

\subsection{Leverage score}
\label{subsec:adaptive_leverage}

Leverage scores measure contributions relative to the geometry
of a linear system and have been widely used in randomized
linear algebra and kernel approximation
\citep{mahoney2011randomized,drineas2012fast,
bach2017equivalence,avron2017random,li2021towards}.
In random feature setting, the same idea gives larger weights to
feature components along directions with smaller regularized
kernel eigenvalues.

More specifically, let $\nu>0$ be the regularization parameter
used for feature scoring. Using the exact finite-bag Gram matrix
would give the ideal score matrix
$\Bb_{\mathrm{LS}}^{\mathrm{ideal}}
=(\Kb_N+\nu\Ib_n)^{-1}$ and scores
\begin{equation}
\label{eq:adaptive_LS_score_ideal}
    q_{\mathrm{LS},m}^{\mathrm{ideal}}
    =
    \widehat{\zb}_m^\top
    \Bb_{\mathrm{LS}}^{\mathrm{ideal}}
    \widehat{\zb}_m.
\end{equation}
The parameter $\nu$ controls the geometry used for selection,
while $\lambda$ in \eqref{eq:adaptive_RaDR_estimator} controls
the subsequent ridge fit.

Since forming $\Kb_N$ requires exact kernel evaluations across
bags, we use the pool approximation $\widehat{\Kb}_0$ from
\eqref{eq:adaptive_pool_gram} for computation.
The practical score matrix is therefore
$\widehat{\Bb}_{\mathrm{LS}}
=(\widehat{\Kb}_0+\nu\Ib_n)^{-1}$, giving
\begin{equation}
\label{eq:adaptive_LS_score_practical}
    q_{\mathrm{LS},m}
    =
    \widehat{\zb}_m^\top
    \widehat{\Bb}_{\mathrm{LS}}\widehat{\zb}_m.
\end{equation}
All scores can be evaluated through
\eqref{eq:adaptive_score_vector}, using
$\Bb=\widehat{\Bb}_{\mathrm{LS}}$.

This construction uses the full candidate pool to estimate the
kernel geometry before selecting features.
Once $\widehat{\Bb}_{\mathrm{LS}}$ is formed, it remains fixed
across subsets, so the realized-pool optimality result in
Proposition~\ref{prop:adaptive_trace_selection} applies.
For the ideal matrix $\Bb_{\mathrm{LS}}^{\mathrm{ideal}}$,
the conditional expectation identity in
\eqref{eq:adaptive_expected_dominance} also applies.

\subsection{Distribution variance score}
\label{subsec:adaptive_distribution_variance}

Adaptive kernel learning for distribution regression has also
been studied from a geometric perspective
\citep{kachaiev2024learning}. A natural objective is to obtain
representations that separate the observed distributions.
Within A-RaDR, this leads to selecting features whose empirical
averages vary substantially across bags.

To formalize this idea, we first recall the distribution
variance of kernel mean embeddings
\citep{muandet2013domain}.

\begin{definition}
For a collection $\mathcal M=\{\mathbb P_1,\ldots,\mathbb P_n\}$,
let $\bar\mu=n^{-1}\sum_{i=1}^n\mu_{\mathbb P_i}$.
The distribution variance is
\begin{equation}
\label{eq:distribution_variance}
\begin{aligned}
    \mathbb V_{\mathcal H_k}(\mathcal M)
    &=
    \frac{1}{n}\sum_{i=1}^n
    \|\mu_{\mathbb P_i}-\bar\mu\|_{\mathcal H_k}^2\\
    &=
    \frac{1}{n}\mathrm{tr}(\Kb)
    -
    \frac{1}{n^2}\oneb^\top\Kb\oneb,
\end{aligned}
\end{equation}
where $[\Kb]_{ij}
=\langle\mu_{\mathbb P_i},\mu_{\mathbb P_j}\rangle_{\mathcal H_k}$
and $\oneb\in\mathbb R^n$ is the all-ones vector.
\end{definition}

Notice that with the centering matrix
$\Hb=\Ib_n-n^{-1}\oneb\oneb^\top$, the variance becomes
$n^{-1}\mathrm{tr}(\Kb\Hb)$, akin to the variance formulation in kernel PCA \citep{wang2021orcca}.
Since the factor $1/n$ does not affect the maximizing subset,
we choose $\Bb_{\mathrm{DV}}=\Hb$, which gives
\begin{equation}
\label{eq:adaptive_DV_trace}
    \mathcal J_{\mathrm{DV}}(\Kb)
    =
    \mathrm{tr}(\Kb\Hb).
\end{equation}
The same identity applies to the selected empirical embeddings:
\[
    \frac{1}{n}\mathcal J_{\mathrm{DV}}(\widehat{\Kb}_S)
    =
    \frac{1}{n}\sum_{i=1}^n
    \left\|
        \widehat{\mub}_{i,S}
        -
        \frac{1}{n}\sum_{j=1}^n\widehat{\mub}_{j,S}
    \right\|_2^2.
\]
Thus, maximizing the trace objective maximizes the empirical
variance of the selected bag representations.

For convenience, denote the average of candidate $m$ across
bags by $\bar z_m=n^{-1}\oneb^\top\widehat{\zb}_m$.
Its score can then be written as
\begin{equation}
\label{eq:adaptive_DV_score}
    q_{\mathrm{DV},m}
    =
    \widehat{\zb}_m^\top\Hb\widehat{\zb}_m
    =
    \sum_{i=1}^n
    \big(\widehat z_i(\widetilde\eta_m)-\bar z_m\big)^2.
\end{equation}
Consequently, this criterion retains features that maximizes the cross-bag embedding variation.

\subsection{Unified view of adaptive RaDR}
\label{subsec:adaptive_unified_view}

The three criteria fit into the same selection procedure through
the following choices of score matrix:
\begin{equation}
\label{eq:adaptive_B_summary}
\begin{array}{c|c|c}
\text{Variant} & \text{Score matrix} & \text{Candidate score}\\[3pt]
\hline
\text{Kernel alignment}
&
\Bb_{\mathrm{KA}}=\yb\yb^\top
&
(\yb^\top\widehat{\zb}_m)^2
\\[5pt]
\text{Leverage score}
&
\widehat{\Bb}_{\mathrm{LS}}
=(\widehat{\Kb}_0+\nu\Ib_n)^{-1}
&
\widehat{\zb}_m^\top
\widehat{\Bb}_{\mathrm{LS}}\widehat{\zb}_m
\\[5pt]
\text{Distribution variance}
&
\Bb_{\mathrm{DV}}=\Hb
&
\widehat{\zb}_m^\top\Hb\widehat{\zb}_m
\end{array}
\end{equation}
In each case, the selected subset follows
\eqref{eq:adaptive_selected_set}, and the predictor follows
\eqref{eq:adaptive_RaDR_estimator}.
Kernel alignment adapts to the responses, leverage scores adapt
to the regularized kernel geometry, and distribution variance
adapts to the spread of the bag representations.

Through these choices, A-RaDR uses the observed training data to
allocate a finite feature budget toward the chosen objective.
The selected features then define the kernel used for both
training and prediction.

\subsection{Computational complexity}
\label{subsec:adaptive_complexity}

The additional computation in A-RaDR comes from
scoring the candidate pool, followed by feature selection. For instances in $\mathbb R^d$, constructing all candidate bag averages costs
$\Oc(nNM_0)$ using direct inner products.
Once these averages are available, kernel alignment and
distribution variance each require $\Oc(nM_0)$ operations to
evaluate all scores.

Leverage scores require a regularized linear system involving
the candidate-pool kernel. Working in the bag space gives a
cost of $\Oc(n^2M_0+n^3)$ for forming the kernel and evaluating
all scores. When $M_0<n$, the Woodbury inversion lemma can reduce the score computational complexity to
$\Oc(nM_0^2+M_0^3)$ operations. Therefore, leverage score cost is
$
    \Oc\!\left(
        \min\{n^2M_0+n^3,\;nM_0^2+M_0^3\}
    \right).
$

The subsequent ridge fit has the same computational cost as RaDR with $D$ features, $\Oc(nD^2+D^3)$ in the primal form or $\Oc(n^2D+n^3)$ in the dual form.

\section{Empirical Results}
\label{sec:experiments}

In this section, we evaluate RaDR and A-RaDR variants across a broad collection of synthetic and real-world distribution regression tasks. We carry out three sets of simulations. The first set of simulations is designed to examine the behavior of RaDR/A-RaDR vs vanilla kernel distribution regression under a controlled synthetic setting. The second set of simulations is designed to investigate the impact of training and testing bag sizes on kernel distribution regression and RaDR framework. The last set of simulations is the large scale real-world dataset benchmark, where we will compare the RaDR and A-RaDR frameworks to state-of-the-art distribution regression methods in both kernel and deep learning regime. All simulations carried out in this section focus on classification, which is the predominant distribution regression target problem in the literature \citep{kachaiev2024learning} and across our $41$ real-world examples.

% To address these questions, we begin with a sequence of controlled synthetic experiments designed to isolate the individual components of the RaDR framework. We first study the generalization behavior of vanilla RaDR relative to exact kernel mean embedding methods as the number of random features varies. We then investigate the role of adaptive feature selection by analyzing how the candidate feature pool size influences downstream performance after score-based selection. 

% Next, we examine the interaction between bag sampling and random feature approximation through a series of experiments that vary both training and testing bag sizes. These studies include both fixed-size and variable-size bag settings, allowing us to evaluate RaDR under more realistic distribution sampling scenarios. Finally, we analyze the joint effect of training bag size and random feature count, highlighting the trade-off between distribution estimation quality and feature approximation quality.

% After establishing these behaviors in controlled settings, we turn to large-scale real-world evaluation on more than 30 real-world benchmark datasets. We compare RaDR and A-RaDR against exact kernel distribution regression methods as well as modern deep set learning architectures, including Deep Sets and Set Transformers. These experiments are designed to evaluate not only predictive performance, but also computational scalability across diverse problem regimes. We focus on reporting classification accuracy in the main text, while F1 scores are deferred to the appendix for completeness.

\subsection{A taxonomy of benchmark algorithms}
\label{subsec:benchmark_taxonomy}

To provide a comprehensive evaluation of RaDR and A-RaDR, we benchmark against a diverse collection of distribution regression algorithms spanning classical kernel methods, deep learning architectures for set-valued inputs, and the random feature based approaches. The benchmark methods represent the major paradigms currently used for learning from distributions and sets.

The first group consists of kernel mean embedding (KME) based distribution regression methods. We include the standard two-stage kernel distribution regression framework of \citep{szabo2016learning}. Specifically, \textbf{KME\_Gaussian} employs Gaussian kernels for both the embedding and distribution regression stages, while \textbf{KME\_Linear} uses a Gaussian embedding kernel together with a linear distribution kernel. We further include two adaptive kernel learning approaches proposed in \citet{kachaiev2024learning}. The first, denoted by \textbf{AKME\_MDKE}, learns the embedding kernel using the Maximum Distribution Kernel Entropy (MDKE) criterion \citep{kachaiev2024learning}. The second, denoted by \textbf{AKME\_DV}, learns the embedding kernel using the Distribution Variance (DV) criterion. For both adaptive approaches, we consider linear and Gaussian distribution kernels, resulting in the variants \textbf{AKME\_MDKE\_Linear}, \textbf{AKME\_MDKE\_Gaussian}, \textbf{AKME\_DV\_Linear}, and \textbf{AKME\_DV\_Gaussian}.

The second group consists of deep learning architectures designed for set-valued inputs. We consider \textbf{Deep Sets (DeS)} \citep{zaheer2017deep}, a permutation-invariant neural architecture that aggregates instance-level representations through a symmetric pooling operation. We also consider the \textbf{Set Transformer (SeT)} \citep{lee2019set}, which leverages self-attention mechanisms to capture interactions among instances within a bag. Two Set Transformer variants are included: \textbf{SeT\_SAB}, which employs Self-Attention Blocks (SAB), and \textbf{SeT\_ISAB}, which utilizes Induced Self-Attention Blocks (ISAB) to improve scalability on larger sets.

The final group consists of the random feature based methods. \textbf{RaDR} denotes the vanilla random feature approximation of kernel distribution regression introduced in Section~\ref{sec:RaDR}. We further consider three adaptive variants developed in Section~\ref{sec:A-RaDR}: \textbf{A-RaDR\_KA}, which performs kernel-alignment-based feature selection; \textbf{A-RaDR\_LS}, which utilizes leverage-score-based feature selection; and \textbf{A-RaDR\_DV}, which employs the distribution variance criterion. Similar to the KME-based methods, each RaDR and A-RaDR algorithm is evaluated using both linear and Gaussian second-stage distribution kernels.

Table~\ref{tab:benchmark_taxonomy} summarizes the complete benchmark suite used throughout the synthetic and real-data experiments. Note that the kernel column describes the distribution kernel, embedding kernel is always chosen to be Gaussian kernel.

\begin{table*}[h]
\centering
\caption{Summary of benchmark algorithms used throughout the empirical study.}
\label{tab:benchmark_taxonomy}
% \small
\vspace{10pt}
\begin{tabular}{lccccc}
\hline
Algorithm &
KME &
RF &
Deep &
Kernel &
Reference \\
\hline

KME\_Linear &
\checkmark &
 &
 &
Linear &
\citep{szabo2016learning}
\\

KME\_Gaussian &
\checkmark &
 &
 &
Gaussian &
\citep{szabo2016learning}
\\

AKME\_DV\_Linear &
\checkmark &
 &
 &
Linear &
\citep{kachaiev2024learning}
\\

AKME\_DV\_Gaussian &
\checkmark &
 &
 &
Gaussian &
\citep{kachaiev2024learning}
\\

AKME\_MDKE\_Linear &
\checkmark &
 &
 &
Linear &
\citep{kachaiev2024learning}
\\

AKME\_MDKE\_Gaussian &
\checkmark &
 &
 &
Gaussian &
\citep{kachaiev2024learning}
\\

\hline

DeS &
 &
 &
\checkmark &
N/A &
\citep{zaheer2017deep}
\\

SeT\_SAB &
 &
 &
\checkmark &
N/A &
\citep{lee2019set}
\\

SeT\_ISAB &
 &
 &
\checkmark &
N/A &
\citep{lee2019set}
\\

\hline

RaDR\_Linear &
 &
\checkmark &
 &
Linear &
This paper
\\

RaDR\_Gaussian &
 &
\checkmark &
 &
Gaussian &
This paper
\\

A-RaDR\_KA\_Linear &
 &
\checkmark &
 &
Linear &
This paper
\\

A-RaDR\_KA\_Gaussian &
 &
\checkmark &
 &
Gaussian &
This paper
\\

A-RaDR\_LS\_Linear &
 &
\checkmark &
 &
Linear &
This paper
\\

A-RaDR\_LS\_Gaussian &
 &
\checkmark &
 &
Gaussian &
This paper
\\

A-RaDR\_DV\_Linear &
 &
\checkmark &
 &
Linear &
This paper
\\

A-RaDR\_DV\_Gaussian &
 &
\checkmark &
 &
Gaussian &
This paper
\\

\hline
\end{tabular}
\end{table*}

\subsection{Experimental Settings}

\paragraph{Synthetic Dataset}

The synthetic benchmark is designed as a distribution-level binary classification problem, where each observation is a bag of exchangeable samples. For bag $i$, we first draw a latent class variable $s_i \in \{-1,1\}$. We then draw a bag-specific nuisance vector
\[
    \bb_i \sim \mathcal{N}(0,\sigma_b^2 \Ib_8),
\]
from an isotropic Gaussian distribution of dimension $8$ with noise scale $\sigma_b = 0.8$, which induces random between-bag variation unrelated to the class signal. Conditioned on $s_i$ and $\bb_i$, the informative coordinates of the bag distribution are shifted according to
\[
    \mub_i
    =
    s_i \alpha \eb_1
    + \frac{1}{2}s_i \alpha \eb_2
    + \frac{1}{2} \bb_i,
\]
where $\eb_1,\eb_2$ are coordinate basis vectors pointing to the first and second elements which are the informative elements, and $\alpha=0.7$ controls the strength of the class-dependent mean shift. The instances in bag $i$ are sampled independently as
\[
    \xb_{ij} \sim \mathcal{N}(\mub_i,\sigma_w^2 \Ib_8),
    \qquad j=1,\ldots,N_i,
\]
where $N_i$ is the bag size specified by the corresponding experimental setting, and $\sigma_w = 1.6$ represents the within group sample noise level. To make the problem less separable from simple sample means, additional independent noise with standard deviation $1$ is injected into the non-informative coordinates, i.e., element $\#3$ to $\#8$. Finally, the observed bag label is generated from a noisy nonlinear functional of the empirical distribution,
\[
    y_i
    =
    \mathbf{1}
    \left\{
        \frac{1}{N_i}\sum_{j=1}^{N_i} \xb_{ij,1}
        +
        0.35\,\frac{1}{N_i}\sum_{j=1}^{N_i} \tanh(\xb_{ij,2})
        +
        0.15\,\varepsilon_i
        >
        0.2\,\xi_i
    \right\},
\]
with $\varepsilon_i,\xi_i \sim \mathcal{N}(0,1)$ independently, and $\mathbf{1}$ denotes the indicator function. $\xb_{ij,1}$ denotes the first element in vector $\xb_{ij}$. Training and test bags are generated independently from the same data-generating process. In all synthetic simulations, the true labels $y_i$ are first generated with $500$ vectors by default, and the actual training and testing bags are down-sampled from these $500$ vectors if a smaller bag size is required. The total number of training bags is set to 500 by default.

\paragraph{General Experimental Settings}

The KME and RaDR algorithms tune the hyperparameters including the embedding and distribution kernel bandwidth, and the SVM cost parameter with $3$-fold cross validation on the training sets. For both embedding kernel and distribution kernel bandwidth, the CV list is set to $[0.25,0.5,1,2,4]$. The SVM cost parameter CV list is set to $[0.1,1,10,100]$. For A-RaDR, the embedding and distribution kernel bandwidth is set to $1$ by default, and only the SVM cost parameter is cross-validated with the same precision.

The AKME methods adopt the kernel optimization approach proposed in \citet{kachaiev2024learning}. The kernel trajectory is altered via a neural network encoder $f_{\theta}: \R^d \rightarrow \R^{16}$, a one-hidden-layer ReLU neural network with fixed hidden dimensions throughout all the simulations. The encoder network is trained via an Adam optimizer \citep{kingma2014adam}, a learning rate of $5\times10^{-4}$, a bag size of 16, and a total of 150 training epochs.

The DeS model uses a neural network encoder $f_{\theta}: \R^d \rightarrow \R^{50}$ with one hidden layer and ELU activation as recommended by the original paper \citep{zaheer2017deep}. The classifier is a two hidden layer ELU neural network. For the real-world dataset simulation, there are roughly 30k trainable parameters for each dataset (small variations depending on the input dimension). The whole network (encoder + classifier) is trained via an Adam optimizer, a learning rate of $10^{-3}$, a bag size of 16, and a total of 150 training epochs.

The SeT methods use a Transformer architecture with 128 hidden dimensions, 4 attention heads, 2 encoder layers and 2 decoder layers, and layer norm is disabled. The classifier is a linear classification layer. For the real-world dataset simulation , there are around 300k trainable parameters for the SAB variant and 450k trainable parameters for the ISAB variant for each dataset with small variations. The whole network (encoder + classifier) is trained via an Adam optimizer, a learning rate of $10^{-3}$, a bag size of 16, and a total of 150 training epochs.

All the simulations included in this section are repeated over $10$ Monte-Carlo runs to create error bars. The Des \citep{zaheer2017deep} and SeT \citep{lee2019set} implementations follow the open access repositories published by the authors. When recording training time, we include the tuning time for KME, RaDR, and A-RaDR algorithms as well for fair comparison against the deep learning counterparts. We use classification accuracy as the primary performance metric throughout the experiments, defined as
\begin{equation*}
    \mathrm{Accuracy}
    =
    \frac{\text{Number of Correctly Classified Bags}}
         {\text{Total Number of Bags}}.
\end{equation*}
We omit the F1 score, as it exhibits the same overall trends as accuracy across all our simulations.

\subsection{RaDR and A-RaDR ablation studies}
\label{subsec:RaDR_ablation}

In this subsection, we investigate the empirical behavior of RaDR and A-RaDR on synthetic distribution regression problems. The goal of these experiments is twofold. First, we examine whether random feature based distribution regression can recover the predictive performance of exact kernel distribution regression while substantially reducing computational cost. Second, we study the effectiveness of adaptive random feature selection and quantify how the candidate feature pool size influences the resulting performance gains.

Throughout this subsection, we focus on the comparison between RaDR, A-RaDR, and the corresponding exact kernel mean embedding (KME) methods. Unless otherwise stated, all results are reported over both linear and Gaussian second-stage distribution kernels.

\subsubsection{RaDR and A-RaDR versus exact kernel distribution regression}
\label{subsubsec:RaDR_vs_kme}

\begin{figure}[h]
    \centering
    \includegraphics[width=0.9\linewidth]{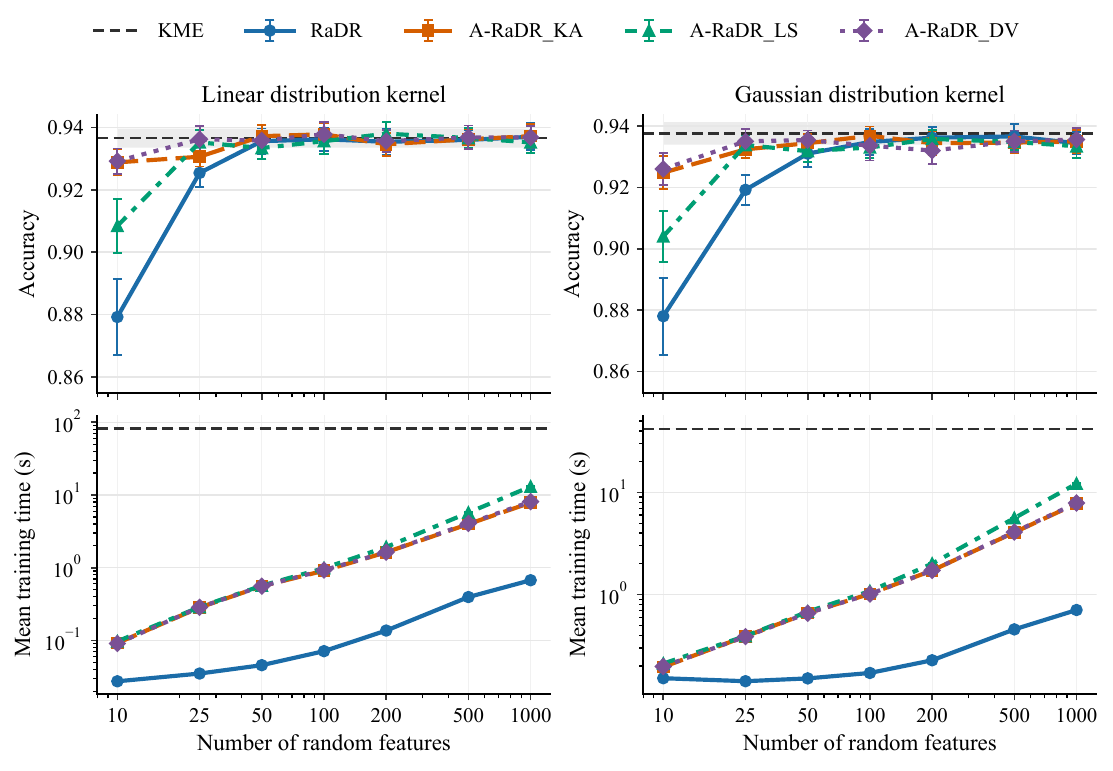}
    \caption{Accuracy and training time of RaDR and A-RaDR as the number of random features varies. The dashed horizontal line denotes exact KME-based distribution regression.}
    \label{fig:rf_feature_ablation}
\end{figure}

Figure~\ref{fig:rf_feature_ablation} compares the predictive accuracy and training time of RaDR and A-RaDR as the number of random features increases. The dashed horizontal line corresponds to the performance of exact kernel distribution regression, while the remaining curves represent random feature approximations with varying feature budgets.

Several observations can be made. First, RaDR exhibits a clear approximation-improvement pattern. When only a small number of random features are used, the approximation error introduced by random features leads to a noticeable loss of predictive accuracy. However, the performance rapidly improves as additional random features are introduced and eventually approaches that of exact kernel distribution regression. This behavior is observed consistently under both linear and Gaussian distribution kernels.

Second, adaptive random feature selection substantially improves performance in the low-feature regime. When only a small number of random features are available for the second stage, all A-RaDR variants outperform vanilla RaDR, with A-RaDR\_KA and A-RaDR\_DV exhibiting the largest improvements. This demonstrates that selecting informative random features is considerably more effective than uniform random sampling when the feature budget is limited.

% Third, the performance gap between RaDR and A-RaDR gradually disappears as the number of random features increases. Once the random feature budget is sufficiently large, both RaDR and A-RaDR achieve performance levels nearly indistinguishable from exact kernel distribution regression. This observation is consistent with the theoretical analysis developed in Section~\ref{sec:RaDR}, which predicts that the random feature approximation error decreases as the feature dimension increases.

Finally, substantial computational savings are achieved even when predictive performance closely matches that of exact kernel methods. Across both kernel choices, RaDR requires only a fraction of the training time needed by KME while achieving nearly identical predictive accuracy. Although A-RaDR incurs additional computational cost due to the feature-scoring procedure, its runtime remains substantially lower than exact kernel distribution regression while providing noticeable performance gains when the feature budget is small. These results suggest that RaDR offers significant computational efficiency, while A-RaDR offers an attractive trade-off between efficiency and performance.

\subsubsection{Effect of candidate feature pool size}
\label{subsubsec:pool_multiplier}

\begin{figure}[h]
    \centering
    \includegraphics[width=0.8\linewidth]{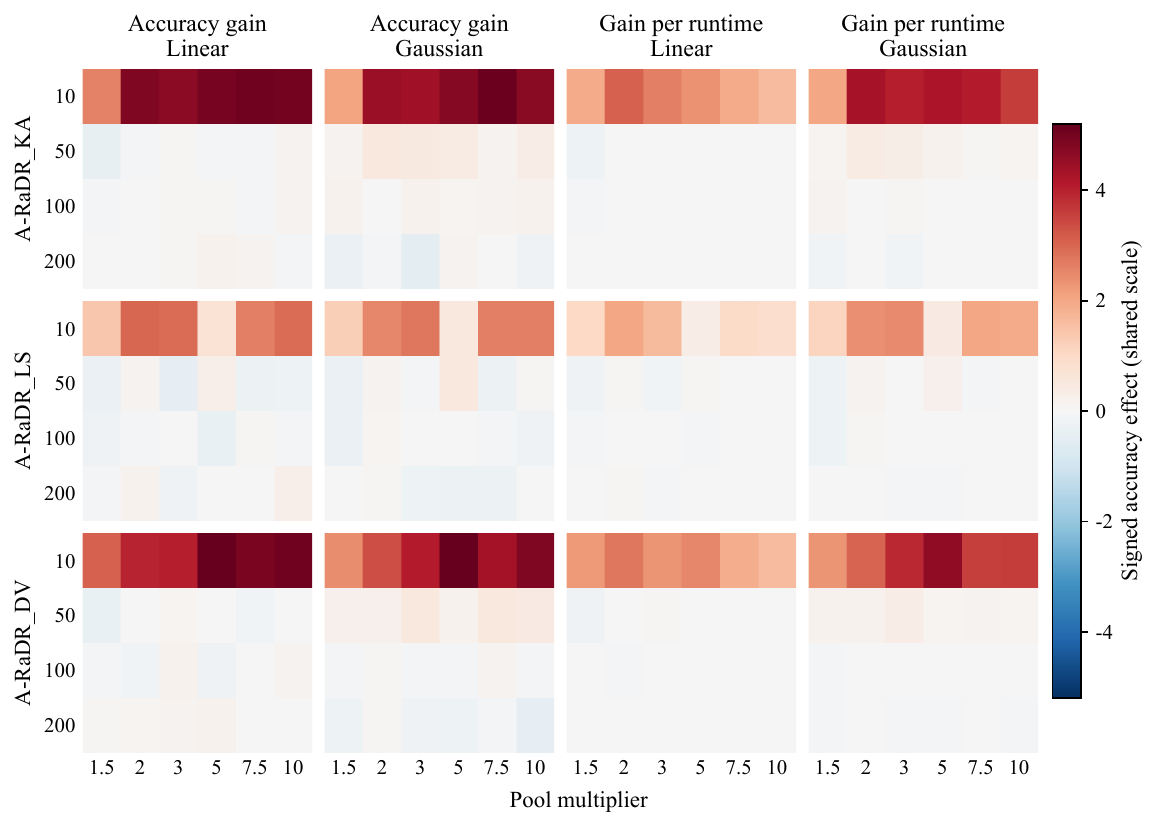}
    \caption{The feature pool multiplier ablation study. \textbf{Left two columns}: Classification accuracy differences between A-RaDR algorithms vs vanilla RaDR algorithm. \textbf{Right two columns}: Accuracy improvements divided by computational overhead (A-RaDR runtime/RaDR runtime). A-RaDR algorithms shows clear benefits in low random feature regime, but the difference diminishes as the random feature count increases.}
    \label{fig:rf_pool_multiplier}
\end{figure}

We next examine the effect of the candidate feature pool size in A-RaDR. Recall that A-RaDR first draws a larger candidate pool of random features, scores these candidates according to a task-dependent criterion, and then retains only a subset of selected features for downstream distribution regression. The pool multiplier controls the ratio between the candidate pool size and the final number of selected random features. A larger pool multiplier gives A-RaDR more candidate features to choose from, but also increases the additional scoring cost.

Figure~\ref{fig:rf_pool_multiplier} reports the accuracy gain of A-RaDR over vanilla RaDR under different random feature budgets and pool multipliers. The heatmap values are reported in percentage points, with positive values indicating improvement over RaDR.

The main pattern is that adaptive selection is most beneficial when the final random feature budget is small. When only 10 random features are retained, all three A-RaDR variants produce clear gains over RaDR across both linear and Gaussian distribution kernels. In particular, A-RaDR\_KA and A-RaDR\_DV provide the largest improvements, often yielding gains of several percentage points, amounting to more than $30\%$ error reduction. This supports that adaptive feature selection is especially useful in the low-dimensional approximation regime, where each selected feature has a large influence on the resulting distribution representation.

As the number of retained random features increases, the marginal benefit of adaptive selection decreases rapidly. For feature budgets of 50, 100, and 200, the gains over RaDR become much smaller and are often close to zero. This is consistent with the intuition, as RaDR and KME methods have already been tuned for the best performance, there leaves little room for improvements for A-RaDR, especially with sufficiently large random feature budget.

The effect of the pool multiplier also exhibits diminishing returns. Increasing the pool multiplier generally improves performance when the feature budget is very small, since A-RaDR can search over a richer candidate set. However, once normalized by computational overhead, the additional gains become limited. This suggests that a relatively small candidate pool is often sufficient to capture most of the benefit of adaptive random feature selection. In practice, this provides a useful computational trade-off where A-RaDR can improve low-budget random feature approximations without requiring an excessively large candidate pool.

Overall, these ablation results support the intended role of A-RaDR. Adaptive feature selection is not primarily needed when the random feature budget is already large, its value lies in improving the statistical efficiency of small random feature representations with mild computational overhead. This makes A-RaDR particularly useful when computation or memory constraints limit the number of features that can be used in the downstream learning tasks.

\subsection{The role of train and test bag sizes}
\label{subsec:bag_size}

\begin{figure}[h]
    \centering
    \includegraphics[width=0.9\linewidth]{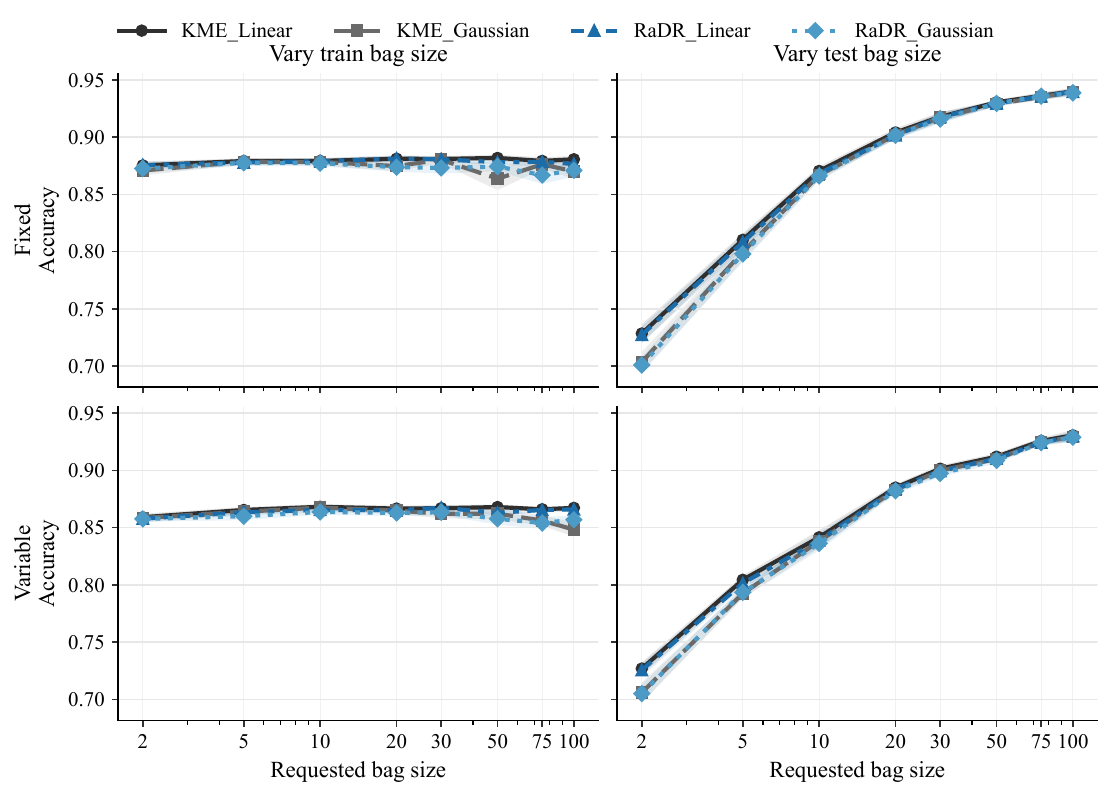}
    \caption{Marginal sensitivity of predictive accuracy to training and testing bag sizes under fixed and variable bag-size settings. Increasing training bag size produces only limited improvements, whereas larger testing bags substantially improve predictive performance. RaDR closely follows the behavior of its exact KME counterparts across all bag-size regimes.}
    \label{fig:bag_size_accuracy_marginal}
\end{figure}

% \begin{figure}[h]
%     \centering
%     \includegraphics[width=\linewidth]{figure/kme_rf_bag_size_accuracy_surfaces.pdf}
%     \caption{Accuracy surfaces as functions of training and testing bag sizes. Across both fixed and variable bag-size settings, predictive performance is dominated by the testing bag size, while the effect of training bag size is comparatively weak. RaDR preserves the same statistical behavior as exact kernel distribution regression throughout the train--test bag-size grid.}
%     \label{fig:bag_size_accuracy_surface}
% \end{figure}

\begin{figure}[h]
    \centering
    \includegraphics[width=\linewidth]{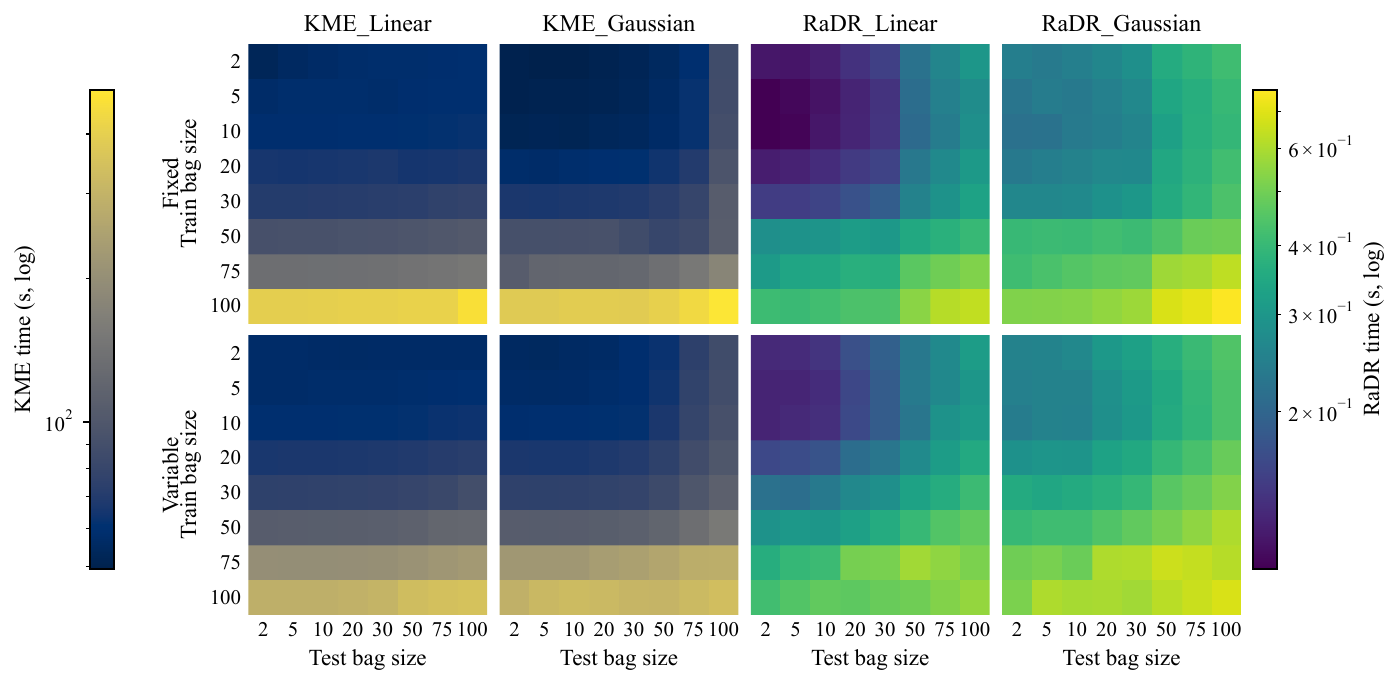}
    \caption{Training-time surfaces as functions of training and testing bag sizes. KME-based distribution regression exhibits strong dependence on training bag size, resulting in highly asymmetric computational scaling. In contrast, RaDR displays nearly symmetric scaling with respect to training and testing bag sizes due to its finite-dimensional random feature representation.}
    \label{fig:bag_size_runtime_surface}
\end{figure}

A unique challenge in distribution regression is the presence of two-stage sampling. Rather than observing distributions directly, each distribution is represented by a finite collection of samples. Consequently, the quality of the learned predictor depends not only on the number of training bags available, but also on the number of observations contained within each bag. Since RaDR is designed specifically to address the computational burden induced by large bags, it is important to understand how training and testing bag sizes influence both predictive performance and computational cost.

To this end, we systematically vary the training and testing bag sizes while keeping all other aspects of the learning problem fixed. We further consider two sampling schemes. In the \emph{fixed} setting, all bags contain exactly the specified number of instances. In the \emph{variable} setting, bag sizes are randomly generated around the requested value, which serves as the average bag size. This allows us to evaluate whether variability in bag cardinality changes the qualitative conclusions.

Figure~\ref{fig:bag_size_accuracy_marginal} reports the marginal effect of training and testing bag sizes on predictive accuracy. An important observation is that increasing the training bag size produces only limited performance improvements. Across all methods, the prediction accuracy shows only mild improvements as the training bag size increases from 2 to 100 observations per bag. In contrast, increasing the testing bag size results in substantial gains in predictive performance. Under both fixed and variable bag-size settings, the accuracy improves steadily as larger testing bags provide increasingly accurate estimates of the underlying test distributions. This observation echoes the well-established phenomenon of kernel distribution regression \citep{szabo2015two}. On the other hand, we observe a clear performance degradation when moving from fixed to variable bag-size settings, this phenomenon has been theoretically investigated in \citet{law2018bayesian}. More importantly, RaDR preserves the same training and testing bag size effects. This shows that the random feature's approximation effect is very much independent from the two-stage sampling process.
% A rigorous theoretical study of this phenomenon represents an interesting direction for future research but lies beyond the scope of this work.

% The same phenomenon becomes even more apparent in the accuracy surfaces shown in Figure~\ref{fig:bag_size_accuracy_surface}. Across all four methods, the accuracy contours are dominated by changes along the testing bag size axis, while variations along the training bag size axis remain comparatively small. In other words, predictive performance is significantly more sensitive to test distribution estimation than to training distribution estimation. Moreover, RaDR closely mirrors the behavior of its exact kernel counterparts throughout the entire train-test bag size grid, demonstrating that the random feature approximation preserves the fundamental statistical characteristics of kernel distribution regression. 

% These results provide an interesting empirical perspective on the two-stage sampling problem. While larger training bags improve the estimation of kernel mean embeddings, the resulting gains quickly saturate once a moderate estimation accuracy has been achieved. In contrast, finite testing bag sizes directly affect the quality of the distribution representation used during prediction and therefore have a much stronger influence on the final prediction accuracy. This asymmetry is consistent with the perturbation analysis developed in Section~\ref{sec:RaDR}, where test distribution estimation error directly enters the prediction process.

While the statistical behavior of RaDR closely matches that of exact kernel methods, the computational scaling characteristics differ substantially. Figure~\ref{fig:bag_size_runtime_surface} presents the corresponding runtime surfaces. For KME-based distribution regression, runtime grows rapidly as the training bag size increases. This behavior reflects the expensive pairwise kernel evaluations required to compare distributions during training. As a result, the computational burden is dominated by the training bag size, producing a highly asymmetric runtime surface.

In contrast, RaDR exhibits a fundamentally different scaling pattern. The runtime surfaces for both RaDR variants become nearly symmetric with respect to training and testing bag sizes. Increasing either the training or testing bag size results in a comparable increase in computational cost. This behavior is a direct consequence of the random feature representation. Once each distribution is summarized through finite-dimensional random feature averages, both training and testing bags contribute through the same linear aggregation procedure, eliminating the quadratic dependence on training bag size that characterizes kernel-based distribution regression.

Taken together, these results highlight an important advantage of the RaDR framework. From a statistical perspective, RaDR faithfully preserves the statistical behavior of exact kernel distribution regression. From a computational perspective, however, RaDR fundamentally alters the scaling law of distribution regression by reducing the dominant training-bag bottleneck to the same linear scaling exhibited by testing bags. 

\subsection{Real-world experiments}
\label{subsec:real-data}

\begin{table*}[h]
\centering
\caption{Dataset information summary for the 41 real-world benchmark datasets.}
\label{tab:data_summary}
\vspace{10pt}
\begin{tabular}{lccccc}
\hline
Name &
Type &
\# Datasets &
Dimension &
Bag Size &
\# Bags\\
\hline

Musk &
molecular activity&
$1$ &
$166$ &
$2-40$ &
$92$ 
\\

Mutagenesis &
molecular activity&
$2$ &
$7$ &
$26-88$ &
$42-188$ 
\\

Protein &
molecular activity&
$1$ &
$8$ &
$35-189$ &
$193$ 
\\

Animal &
Image Annotation&
$3$ &
$230$ &
$1-13$ &
$200$ 
\\

Corel &
Image Annotation&
$20$ &
$9$ &
$2-13$ &
$2000$ 
\\

UCSB Breast Cancer &
Histopathology Analysis&
$1$ &
$708$ &
$21-40$ &
$58$ 
\\

Birds &
Audio Classification&
$13$ &
$38$ &
$2-43$ &
$548$ 
\\

\hline
\end{tabular}
\end{table*}

\begin{table*}[p]
\centering
\caption{Average real-data benchmark performance over 41 complete MIL datasets. Reported values are mean $\pm$ propagated standard error. The top 4 accuracy and the fastest training time are marked in \textbf{bold}.}
\label{tab:mil_real_data_summary}
\vspace{10pt}
\begin{tabular}{lcc}
\hline
Algorithm &
Accuracy &
Training Time (s) \\
\hline

KME\_Linear &
$0.9347 \pm 0.0016$ &
$264.37 \pm 1.30$
\\

KME\_Gaussian &
$\mathbf{0.9398 \pm 0.0015}$ &
$255.16 \pm 2.44$
\\

AKME\_DV\_Linear &
$0.9022 \pm 0.0017$ &
$37.03 \pm 0.26$
\\

AKME\_DV\_Gaussian &
$0.9110 \pm 0.0016$ &
$36.77 \pm 0.29$
\\

AKME\_MDKE\_Linear &
$0.9091 \pm 0.0016$ &
$40.64 \pm 0.26$
\\

AKME\_MDKE\_Gaussian &
$0.9226 \pm 0.0016$ &
$42.29 \pm 0.26$
\\

\hline

DeS &
$0.9266 \pm 0.0017$ &
$41.53 \pm 0.14$
\\

SeT\_SAB &
$0.9356 \pm 0.0016$ &
$110.40 \pm 0.23$
\\

SeT\_ISAB &
$0.9323 \pm 0.0019$ &
$142.14 \pm 0.32$
\\

\hline

RaDR\_Linear &
$0.9291 \pm 0.0017$ &
$0.2534 \pm 0.0013$
\\

RaDR\_Gaussian &
$0.9341 \pm 0.0017$ &
$0.7223 \pm 0.0027$
\\

A-RaDR\_KA\_Linear &
$0.9199 \pm 0.0015$ &
$1.26 \pm 0.01$
\\

A-RaDR\_KA\_Gaussian &
$\mathbf{0.9361 \pm 0.0013}$ &
$1.81 \pm 0.01$
\\

A-RaDR\_LS\_Linear &
$0.9238 \pm 0.0016$ &
$2.40 \pm 0.008$
\\

A-RaDR\_LS\_Gaussian &
$\mathbf{0.9357 \pm 0.0016}$ &
$3.02 \pm 0.007$
\\

A-RaDR\_DV\_Linear &
$0.9256 \pm 0.0016$ &
$1.31 \pm 0.006$
\\

A-RaDR\_DV\_Gaussian &
$\mathbf{0.9360 \pm 0.0015}$ &
$1.80 \pm 0.008$
\\

\hline

% RaDR\_Linear\_low &
% $0.8985 \pm 0.0017$ &
% $\mathbf{0.0350 \pm 0.0001}$
% \\

RaDR\_Gaussian\_low &
$0.9149 \pm 0.0016$ &
$0.4780 \pm 0.0010$
\\

% A-RaDR\_KA\_Linear\_low &
% $0.8988 \pm 0.0016$ &
% $0.1114 \pm 0.0003$
% \\

A-RaDR\_KA\_Gaussian\_low &
$0.9156 \pm 0.0020$ &
$0.5577 \pm 0.0009$
\\

% A-RaDR\_LS\_Linear\_low &
% $0.8971 \pm 0.0019$ &
% $0.1192 \pm 0.0003$
% \\

A-RaDR\_LS\_Gaussian\_low &
$0.9143 \pm 0.0018$ &
$0.5571 \pm 0.0009$
\\

% A-RaDR\_DV\_Linear\_low &
% $0.9009 \pm 0.0018$ &
% $0.1144 \pm 0.0004$
% \\

A-RaDR\_DV\_Gaussian\_low &
$0.9185 \pm 0.0017$ &
$0.5576 \pm 0.0009$
\\

Gaussian\_low\_Average &
$0.9158 \pm 0.0018$ &
$0.5376 \pm 0.0009$
\\

\hline
\end{tabular}
\end{table*}

% \begin{figure}[h]
%     \centering
%     \includegraphics[width=0.95\linewidth]{figure/mil_real_data_accuracy_runtime_by_algorithm.pdf}
%     \caption{Average predictive accuracy (left) and training time (right) across 41 real-world datasets. Error bars represent propagated standard errors across datasets. Striped bars denote low-feature variants used to evaluate the effectiveness of adaptive random feature selection in the limited-feature regime.}
%     \label{fig:mil}
% \end{figure}

\begin{figure}[h]
    \centering
    \includegraphics[width=0.95\linewidth]{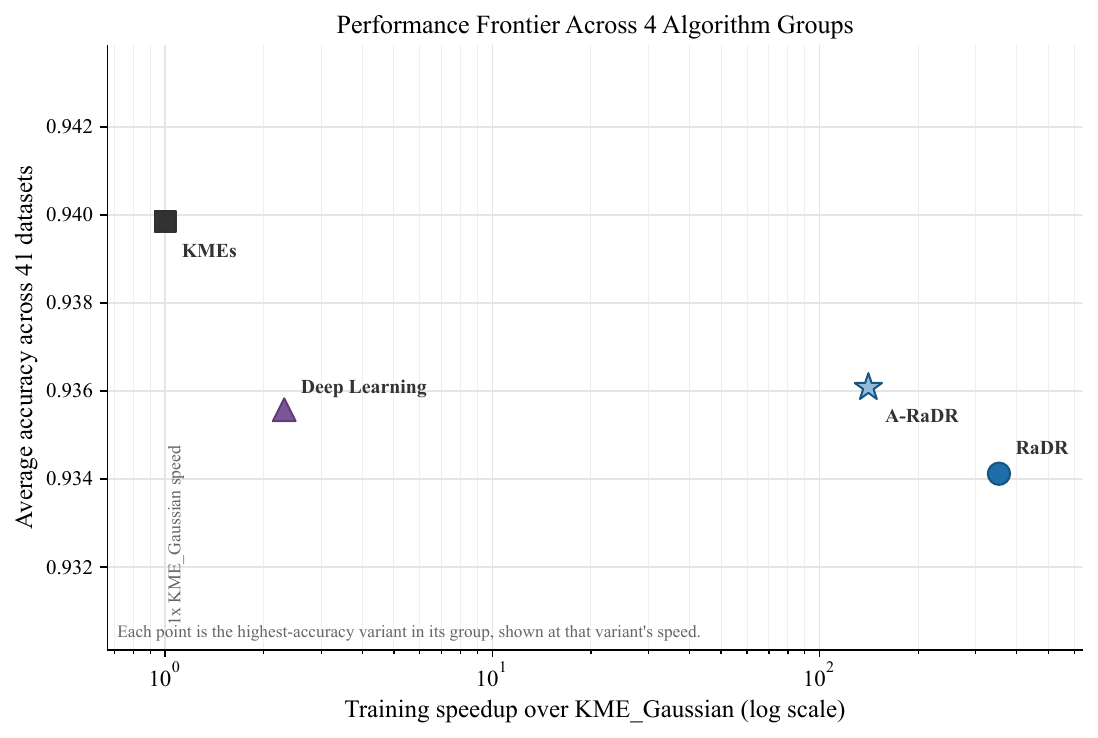}
    \caption{The performance frontier (best accuracy) across the four algorithm groups. X-axis represents the relative training speed up over KME\_Gaussian algorithm, computed as $\text{KME\_Gaussian Runtime}/\text{Algorithm Runtime}$. A-RaDR frontier strikes a balance between computation and performance, ranking second in both criteria.}
    \label{fig:mil_scatter}
\end{figure}

We evaluate RaDR and its adaptive variants on a diverse collection of $41$ real-world multi-instance learning datasets spanning molecular activity prediction, image annotation, histopathology analysis, and audio classification tasks. The benchmark suite includes the Musk and Mutagenesis datasets \citep{dietterich1997solving,srinivasan1995comparing}, the Protein dataset \citep{tao2004svm}, the Elephant, Fox, and Tiger image datasets \citep{andrews2003support}, twenty Corel image annotation datasets \citep{chen2006miles}, the UCSB Breast Cancer dataset \citep{kandemir2014empowering}, and thirteen bird-species audio classification datasets \citep{briggs2012acoustic}. These datasets exhibit substantial diversity in both feature dimensionality and bag structure, ranging from small image collections to large molecular and audio datasets, thereby providing a comprehensive assessment of predictive performance and computational scalability. We exclude the Musk2 dataset from the final benchmark because its extremely large bag sizes (up to $1,044$ instances per bag) cause memory failures for both exact kernel mean embedding methods and Set Transformer baselines. Table~\ref{tab:data_summary} provides a summary for all 41 benchmark datasets.

The benchmark is carried out over the full set of algorithms presented in Table~\ref{tab:benchmark_taxonomy}. Moreover, we introduce a set of low random feature variants to the RaDR and A-RaDR framework to further highlight the efficiency benefit in subsecond territory. Specifically, unlike standard RaDR and A-RaDR algorithms with 256 random features and pool size of 2048, the \_low variants only use 16 random features for the prediction tasks, and 160 random features for the A-RaDR pool.

Figure~\ref{fig:mil_scatter}, and Table~\ref{tab:mil_real_data_summary} summarize the average predictive accuracy and computational cost across all datasets. We point out several important observations. First, with careful tuning, exact kernel distribution regression remains the strongest statistical benchmark. In particular, KME\_Gaussian achieves the highest average predictive accuracy among all methods considered. However, the performance gap between exact kernel methods and the best random-feature-based approaches is small. The strongest A-RaDR variants achieve average accuracies exceeding $0.936$, remaining within less than half a percentage point of KME\_Gaussian while operating in a finite-dimensional representation space. These results indicate that RaDR successfully preserves the predictive advantages of kernel distribution regression despite replacing expensive kernel evaluations with random feature approximations.

More importantly, the adaptive random feature variants emerge as the strongest practical competitors to exact kernel methods. Among all non-KME approaches, A-RaDR\_KA\_Gaussian achieves the highest average accuracy, followed closely by A-RaDR\_DV\_Gaussian and A-RaDR\_LS\_Gaussian. At the same time, these methods require only one to three seconds of average training time, compared with nearly five minutes for exact kernel distribution regression. Consequently, the proposed A-RaDR framework occupies a particularly attractive operating regime, delivering predictive performance that is nearly indistinguishable from the strongest kernel methods while reducing computational cost by over two orders of magnitude.

For adaptive kernel learning methods, although AKME\_DV and AKME\_MDKE attempt to improve performance by directly optimizing kernel hyperparameters, their predictive performance remains noticeably below both RaDR and A-RaDR while still requiring tens of seconds of training time. This observation is particularly interesting for the distribution variance objective. While AKME\_DV directly optimizes the kernel using the same underlying criterion employed by A-RaDR\_DV, the score-based random feature framework produces substantially stronger predictive performance together with dramatically lower computational cost. These results suggest that controlling the effective kernel through adaptive random feature selection may be a more effective strategy than directly optimizing kernel parameters themselves.

The deep learning baselines exhibit competitive predictive performance, with Set Transformer in particular achieving accuracy levels comparable to several A-RaDR variants. However, these gains come at a substantially higher computational cost. Moreover, unlike RaDR, both Set Transformer and kernel-based methods inherit quadratic computation growth with respect to bag size due to the underlying attention or kernel computations. This characteristic limits the ability to trade computation for memory through mini-batching or related implementation techniques, which is one of the primary mechanisms that enables modern deep learning methods to scale. In contrast, RaDR maintains a fixed-dimensional representation throughout training and prediction, resulting in substantially lower memory requirements and more favorable scaling behavior as bag sizes increase.

The low-feature variants provide additional insight into the role of adaptive random feature selection. A-RaDR in low random feature count regime continues to deliver advantages in performance. Moreover, its computational overhead becomes milder in this regime. For example, in the Gaussian-kernel setting, A-RaDR\_KA\_Gaussian\_low and A-RaDR\_DV\_Gaussian\_low achieve statistically meaningful improvements over RaDR\_Gaussian\_low while increasing average training time only marginally, from approximately $0.48$ seconds to $0.56$ seconds. This result demonstrates that adaptive feature selection can provide measurable performance gains in low-feature regimes at an almost negligible additional computational cost.

Taken together, these results provide strong empirical support for the central premise of this paper. Exact kernel distribution regression remains an exceptionally effective statistical learning framework, but its computational requirements become increasingly restrictive on large-scale problems. RaDR preserves nearly all of the predictive performance of kernel distribution regression while reducing computational cost by several orders of magnitude. Furthermore, A-RaDR provides a practical mechanism for further improving predictive performance, particularly when the random feature budget is constrained. Across a diverse collection of 41 real-world datasets, RaDR and A-RaDR frameworks consistently achieve one of the most favorable accuracy-runtime tradeoffs among all competing methods.

\section{Discussion and Conclusion}\label{sec:discussion}

In this paper, we investigated two longstanding challenges in distribution regression: scalability and adaptivity. Building on existing computational formulations, which we refer to as Random Feature Distribution Regression (RaDR), we developed a finite-sample analysis under two-stage sampling. Our analysis separates distribution-level sampling error, finite within-bag sampling error, and random feature approximation error. It also establishes sufficient feature-budget conditions for preserving the order of our exact two-stage excess-risk upper bound. We further developed Adaptive Random Feature Distribution Regression (A-RaDR), a unified score-guided framework that incorporates task-aware feature selection to improve representation quality while maintaining computational efficiency. Through theoretical analysis and extensive empirical studies, we demonstrated that scalable and adaptive distribution regression can be achieved simultaneously within a principled random feature framework.

We believe several interesting directions could follow this work. Extending the excess-risk analysis to data-dependent feature selection would help explain how the adaptation criterion and candidate-pool size affect predictive performance. Further analysis could also address nonlinear distribution kernels and classification, connecting the theoretical guarantees more closely to the prevalent empirical settings in distribution regression. The incremental structure of random feature mean embeddings additionally motivates studying how sampling error, feature selection, and model updates interact in streaming settings. Such developments would extend the present analysis to a broader class of learning objectives and evolving distribution-valued data.

\newpage

\appendix
\section{Technical Details}
\label{app:RaDR_proofs}

This appendix provides the assumptions, auxiliary lemmas, and proofs supporting Section~\ref{sec:RaDR} and \ref{sec:A-RaDR}. Our theoretical analysis is based on the linear outer kernel
\[
K(\mathbb{P},\mathbb{Q})
=
\langle \mu_{\mathbb{P}},\mu_{\mathbb{Q}}\rangle_{\mathcal{H}_k},
\]
for simplicity. Extensions to nonlinear outer kernels require an additional perturbation layer for the outer kernel and we leave that analysis for future works. The one-stage and two-stage kernel regression results is heavily inspired by the proof of \citet{szabo2015two}, and they build the groundwork for our third stage analysis.

\subsection{Assumptions and auxiliary objects}
\label{app:RaDR_assumptions}

Throughout the appendix, the latent pairs
$(\mathbb P_i,y_i)_{i=1}^n$ are independent and identically
distributed according to a probability measure $\rho$.
Its marginal on distributions is denoted by $\rho_X$, and
\[
f_\rho(\mathbb P)=\mathbb E[y\mid\mathbb P].
\]
Conditional on the latent pairs, the observations within each
bag are independent samples from $\mathbb P_i$, and the bags
are conditionally independent. In particular, conditional on
$\mathbb P_i$, the sampling of the observed bag is independent
of $y_i$. We use the balanced setting $N_i=N$.

The random features (aside from selected features in A-RaDR) are independent of the training data. The regularization parameter $\lambda>0$ is considered to be deterministic. Population risks are evaluated on an independent latent pair $(\mathbb P,y)\sim\rho$. All kernels and feature maps are assumed measurable, and the RKHS-valued expectations below are well defined. We first present the set of assumptions.

\begin{assumption}
\label{ass:RaDR_bounded}
There exist constants $\kappa>0$ and $C_y>0$ such that
\[
\sup_{\xb\in\mathcal{X}} k(\xb,\xb)\le \kappa^2,
\qquad
|y|\le C_y
\quad\text{almost surely}.
\]
\end{assumption}

\begin{assumption}
\label{ass:RaDR_spectral}
The instance domain $\mathcal X\subset\mathbb R^d$ is compact,
with diameter
\[
\ell_{\mathcal X}
=\sup_{\xb,\zb\in\mathcal X}\|\xb-\zb\|_2
\in(0,\infty).
\]
The base kernel is continuous, real-valued, positive definite,
and shift-invariant, with $k(\xb,\xb)=1$ and spectral
probability measure $\Lambda$. Moreover,
\[
0<\sigma_p^2
=\int_{\mathbb R^d}\|\omegab\|_2^2\,d\Lambda(\omegab)
<\infty.
\]
The dimension, kernel parameters, and domain are considered fixed in the asymptotic statements.
\end{assumption}

\begin{assumption}
\label{ass:RaDR_source_capacity}
Let $T:\mathcal{H}_K\to\mathcal{H}_K$ denote the covariance operator
\begin{equation}
\label{eq:RaDR_cov_operator_app}
    Tf
    =
    \int_{\mathcal{M}_1^+(\mathcal{X})}
    \langle f,K_{\mathbb{P}}\rangle_{\mathcal{H}_K}
    K_{\mathbb{P}}
    \,d\rho_X(\mathbb{P}),
\end{equation}
where $K_{\mathbb{P}}=K(\mathbb{P},\cdot)$. There exist $c\in(1,2]$, $b>1$, $R>0$, and $\beta>0$ such that
\[
f_\rho
=
T^{(c-1)/2}g
\quad
\text{for some }g\in\mathcal{H}_K
\text{ with }\|g\|_{\mathcal{H}_K}^2\le R,
\]
and the effective dimension
\[
\mathcal{N}(\lambda)
=
\mathrm{tr}\{(T+\lambda I)^{-1}T\}
\]
satisfies
\[
\mathcal{N}(\lambda)\le \beta\lambda^{-1/b}.
\]
\end{assumption}

\begin{remark}
The compact-domain and spectral assumptions are required for the
uniform RF approximation result. The exact-kernel finite-bag
perturbation argument needs only boundedness and the stated
sampling model. The source-capacity condition is required to establish a one-stage statistical baseline.
\end{remark}

\subsection{One-stage exact-kernel baseline}
\label{app:RaDR_baselines}

The perturbation analysis requires establishing a one-stage KRR bound as an input. The one-stage estimator observes the true distribution embeddings and is defined by
\[
\widehat f_\lambda
=\argmin_{f\in\mathcal H_K}
\left\{
\frac1n\sum_{i=1}^n
\bigl(f(\mathbb P_i)-y_i\bigr)^2
+\lambda\|f\|_{\mathcal H_K}^2
\right\}.
\]

\begin{assumption}
\label{ass:RaDR_baseline_exact}
For the sample sizes, regularization parameters, and confidence
levels under consideration, there is a nonnegative function
$B_{\mathrm{1s}}(n,\lambda,\delta)$ such that
\begin{equation}
\label{eq:RaDR_baseline_exact_app}
\mathbb P\left\{
\mathcal E(\widehat f_\lambda,f_\rho)
\le B_{\mathrm{1s}}(n,\lambda,\delta)
\right\}\ge1-\delta.
\end{equation}
We further assume that any assumptions that establish the above bound will hold in the proceeding perturbation analysis, such that the results will carry over to the two-stage risk bound.
\end{assumption}

\begin{remark}
Under Assumptions~\ref{ass:RaDR_bounded}
and~\ref{ass:RaDR_source_capacity}, the one-stage KRR analysis
of \citet{caponnetto2007optimal}, with the corrections in
Sutherland's \emph{Fixing an Error in Caponnetto and de Vito
(2007)} (\href{https://arxiv.org/abs/1702.02982}{arXiv:1702.02982}),
gives the following bound. There are constants
$C_0,C_1>0$, depending only on the fixed problem parameters,
such that
\[
\begin{split}
B_{\mathrm{1s}}(n,\lambda,\delta)
=C_1\log^2(6/\delta)\bigg[
&\lambda^c
+\frac{\lambda^{c-2}}{n^2}
+\frac{\lambda^{c-1}}{n}+\frac{1}{n^2\lambda}
+\frac{1}{n\lambda^{1/b}}
\bigg],
\end{split}
\]
provided that
\[
0<\lambda\le\min\{1,\|T\|_{\mathrm{op}}\},
\qquad
n\ge C_0\log^2(6/\delta)\lambda^{-(1+1/b)}.
\]
The constants absorb the source radius, the
effective-dimension constant, the kernel bound, and the
bounded-output noise constants.

For fixed $\delta$ and a nonzero covariance operator,
the choice
\[
\lambda_n\asymp n^{-b/(bc+1)}
\]
satisfies the sample-size condition for all sufficiently
large $n$, since $c>1$, and yields
\[
B_{\mathrm{1s}}(n,\lambda_n,\delta)
=O\!\left(n^{-bc/(bc+1)}\right).
\]
This is one example of a practical $B_{\mathrm{1s}}(n,\lambda_n,\delta)$ bound that is compatible with our subsequent perturbation analysis. Other valid one-stage baselines can also be used.
\end{remark}

\subsection{Uniform random feature approximation}
\label{app:RaDR_uniform_rf}

Write $L=\sigma_p\ell_{\mathcal X}>0$. We use the normalized
random-bias cosine map in \eqref{eq:rf_map_sec3}. For simplicity, we first define the error term $\epsilon_D(\delta)$.

\begin{definition}
\label{def:RaDR_epsD}
For $\delta\in(0,1)$ and $D\ge1$, define
\begin{equation}
\label{eq:epsD_app}
\epsilon_D(\delta)
=
\inf\left(
\{3\}\cup
\left\{
t\in(0,\min\{L,3\}]:
98\left(\frac{L}{t}\right)^2
\exp\left(-\frac{Dt^2}{32(d+1)}\right)
\le\delta
\right\}
\right).
\end{equation}
\end{definition}

\begin{theorem}
\label{thm:uniform_rf_app}
Under Assumption~\ref{ass:RaDR_spectral},
\[
\mathbb P\left\{
\sup_{\xb,\zb\in\mathcal X}
|k_D(\xb,\zb)-k(\xb,\zb)|
\le\epsilon_D(\delta)
\right\}\ge1-\delta.
\]
For every fixed $\delta\in(0,1)$, there is a constant
$C_\delta>0$, depending only on $\delta$ and the fixed
kernel and domain parameters, such that, for all
sufficiently large $D$,
\[
\epsilon_D(\delta)^2
\le C_\delta\frac{\log(eD)}{D}.
\]
\end{theorem}

\begin{proof}
Proposition~2 of \citet{sutherland2015error}, applied to
the random-bias cosine construction, gives
\begin{equation}
\label{eq:ss_uniform_bound_app}
\begin{split}
\mathbb P\left\{
\sup_{\xb,\zb\in\mathcal X}
|k_D(\xb,\zb)-k(\xb,\zb)|\ge t
\right\}
&\le
98\left(\frac{L}{t}\right)^2
\exp\left(-\frac{Dt^2}{32(d+1)}\right),\\
&\hspace{20mm}0<t\le L.
\end{split}
\end{equation}
The right-hand side is continuous and decreasing in $t$,
and diverges as $t\downarrow0$. Therefore, if
$\epsilon_D(\delta)<3$, the infimum in
\eqref{eq:epsD_app} is attained by the probability claim in
 \eqref{eq:ss_uniform_bound_app}.

For the edge case where $\epsilon_D(\delta)=3$, observe that
\[
\|\phib_D(\xb)\|_2\le\sqrt2,
\qquad
|k_D(\xb,\zb)|\le2,
\qquad
|k(\xb,\zb)|\le1.
\]
Consequently,
\[
\sup_{\xb,\zb\in\mathcal X}
|k_D(\xb,\zb)-k(\xb,\zb)|\le3
\]
deterministically.

For the convergence rate of $\epsilon_D(\delta)$, choose $A>\sqrt{32(d+1)}$ and set
\[
t_D=A\sqrt{\frac{\log(eD)}{D}},
\qquad
q=\frac{A^2}{32(d+1)}>1.
\]
For sufficiently large $D$, $t_D\le\min\{L,3\}$, and
\[
98\left(\frac{L}{t_D}\right)^2
\exp\left(-\frac{Dt_D^2}{32(d+1)}\right)
=
\frac{98L^2e^{-q}}{A^2}
\frac{D^{1-q}}{\log(eD)}
\longrightarrow0.
\]
Hence $t_D$ is a valid upper bound for all sufficiently large $D$, which proves the stated rate.
\end{proof}

\subsection{Concentration of empirical mean embeddings}
\label{app:RaDR_bag_concentration}

\begin{lemma}
\label{lem:mean_concentration_app}
Under Assumption~\ref{ass:RaDR_bounded}, there is a universal
constant $C_\mu>0$ such that, for every distribution
$\mathbb P$ and $\delta\in(0,1)$,
\[
\mathbb P\left\{
\|\widehat\mu_{\mathbb P}-\mu_{\mathbb P}\|_{\mathcal H_k}
\le C_\mu\kappa\sqrt{\frac{\log(2/\delta)}{N}}
\right\}\ge1-\delta.
\]
Consequently, with probability at least $1-\delta$,
\begin{equation}
\label{eq:uniform_bag_concentration_app}
\max_{1\le i\le n}
\|\widehat\mu_i-\mu_{\mathbb P_i}\|_{\mathcal H_k}
\le C_\mu\kappa\sqrt{\frac{\log(2n/\delta)}{N}}.
\end{equation}
\end{lemma}

\begin{proof}
Let $F=\|\widehat\mu_{\mathbb P}-\mu_{\mathbb P}\|_{\mathcal H_k}$.
Independence and centering give
\[
\mathbb EF^2
=\frac1N\mathbb E
\|k(\xb,\cdot)-\mu_{\mathbb P}\|_{\mathcal H_k}^2
\le\frac{\kappa^2}{N},
\]
so $\mathbb EF\le\kappa/\sqrt N$.
Replacing one observation changes $F$ by at most $2\kappa/N$.
The bounded-differences inequality therefore gives, with
probability at least $1-\delta$,
\[
F\le\frac{\kappa}{\sqrt N}
+\kappa\sqrt{\frac{2\log(1/\delta)}{N}}
\le C_\mu\kappa\sqrt{\frac{\log(2/\delta)}{N}}.
\]
Therefore, $C_\mu=3$ is a sufficient constant.
Apply this bound conditionally to each training distribution
with failure probability $\delta/n$, and take a union bound gives~\eqref{eq:uniform_bag_concentration_app}.
\end{proof}

\subsection{Kernel perturbation at the bag level}
\label{app:RaDR_kernel_perturbation}

For the rest of the appendix, $\widehat{\mathbb P}_i$ denotes
the empirical measure of training bag $i$. We use
\[
\begin{aligned}
K_N(\mathbb P_i,\mathbb P_j)
&=\langle\widehat\mu_i,\widehat\mu_j\rangle_{\mathcal H_k},
&K_N(\mathbb P_i,\mathbb P)
&=\langle\widehat\mu_i,\mu_{\mathbb P}\rangle_{\mathcal H_k},\\
K_{D,N}(\mathbb P_i,\mathbb P_j)
&=\langle\widehat{\mub}_{i,D},\widehat{\mub}_{j,D}\rangle_2,
&K_{D,N}(\mathbb P_i,\mathbb P)
&=\langle\widehat{\mub}_{i,D},\mub_{\mathbb P,D}\rangle_2.
\end{aligned}
\]
In this part of the analysis, we assume training Gram entries are computed with empirical measures. For test vector entries, only the training argument is empirical unless a finite test bag is explicitly specified.

\begin{lemma}
\label{lem:rf_bag_population_app}
On the event of Theorem~\ref{thm:uniform_rf_app}, for all probability measures $\mathbb{P}$ and $\mathbb{Q}$ on $\mathcal{X}$,
\begin{equation}
\label{eq:rf_bag_population_app}
    |K_D(\mathbb{P},\mathbb{Q})-K(\mathbb{P},\mathbb{Q})|
    \le
    \epsilon_D(\delta).
\end{equation}
\end{lemma}

\begin{proof}
By definition,
\[
K_D(\mathbb{P},\mathbb{Q})
=
\int_{\mathcal{X}}\int_{\mathcal{X}}
k_D(\xb,\zb)
\,d\mathbb{P}(\xb)d\mathbb{Q}(\zb),
\]
and
\[
K(\mathbb{P},\mathbb{Q})
=
\int_{\mathcal{X}}\int_{\mathcal{X}}
k(\xb,\zb)
\,d\mathbb{P}(\xb)d\mathbb{Q}(\zb).
\]
Therefore,
\[
\begin{aligned}
|K_D(\mathbb{P},\mathbb{Q})-K(\mathbb{P},\mathbb{Q})|
&=
\left|
\int_{\mathcal{X}}\int_{\mathcal{X}}
\{k_D(\xb,\zb)-k(\xb,\zb)\}
\,d\mathbb{P}(\xb)d\mathbb{Q}(\zb)
\right| \\
&\le
\int_{\mathcal{X}}\int_{\mathcal{X}}
|k_D(\xb,\zb)-k(\xb,\zb)|
\,d\mathbb{P}(\xb)d\mathbb{Q}(\zb) \\
&\le
\sup_{\xb,\zb\in\mathcal{X}}
|k_D(\xb,\zb)-k(\xb,\zb)| \\
&\le
\epsilon_D(\delta).
\end{aligned}
\]
\end{proof}

\begin{lemma}
\label{lem:two_stage_bag_kernel_app}
On the event of Lemma~\ref{lem:mean_concentration_app}, for all $i,j\in\{1,\ldots,n\}$,
\begin{equation}
\label{eq:two_stage_bag_kernel_app}
    |K_N(\mathbb{P}_i,\mathbb{P}_j)-K(\mathbb{P}_i,\mathbb{P}_j)|
    \le
    2C_\mu\kappa^2
    \sqrt{\frac{\log(2n/\delta)}{N}}.
\end{equation}
Moreover, for any test distribution $\mathbb{P}$,
\begin{equation}
\label{eq:two_stage_test_kernel_app}
    |K_N(\mathbb{P}_i,\mathbb{P})-K(\mathbb{P}_i,\mathbb{P})|
    \le
    C_\mu\kappa^2
    \sqrt{\frac{\log(2n/\delta)}{N}}.
\end{equation}
\end{lemma}

\begin{proof}
For the training Gram entry,
\[
K_N(\mathbb{P}_i,\mathbb{P}_j)-K(\mathbb{P}_i,\mathbb{P}_j)
=
\langle \widehat\mu_i,\widehat\mu_j\rangle_{\mathcal{H}_k}
-
\langle \mu_{\mathbb{P}_i},\mu_{\mathbb{P}_j}\rangle_{\mathcal{H}_k}.
\]
Add and subtract $\langle \widehat\mu_i,\mu_{\mathbb{P}_j}\rangle_{\mathcal{H}_k}$:
\[
\begin{aligned}
&|K_N(\mathbb{P}_i,\mathbb{P}_j)-K(\mathbb{P}_i,\mathbb{P}_j)| \\
&\le
|\langle \widehat\mu_i,\widehat\mu_j-\mu_{\mathbb{P}_j}\rangle_{\mathcal{H}_k}|
+
|\langle \widehat\mu_i-\mu_{\mathbb{P}_i},\mu_{\mathbb{P}_j}\rangle_{\mathcal{H}_k}|.
\end{aligned}
\]
By Cauchy--Schwarz,
\[
|K_N(\mathbb{P}_i,\mathbb{P}_j)-K(\mathbb{P}_i,\mathbb{P}_j)|
\le
\|\widehat\mu_i\|_{\mathcal{H}_k}
\|\widehat\mu_j-\mu_{\mathbb{P}_j}\|_{\mathcal{H}_k}
+
\|\mu_{\mathbb{P}_j}\|_{\mathcal{H}_k}
\|\widehat\mu_i-\mu_{\mathbb{P}_i}\|_{\mathcal{H}_k}.
\]
Since $\|\widehat\mu_i\|_{\mathcal{H}_k}\le \kappa$ and $\|\mu_{\mathbb{P}_j}\|_{\mathcal{H}_k}\le \kappa$, \eqref{eq:uniform_bag_concentration_app} gives
\[
|K_N(\mathbb{P}_i,\mathbb{P}_j)-K(\mathbb{P}_i,\mathbb{P}_j)|
\le
2C_\mu\kappa^2
\sqrt{\frac{\log(2n/\delta)}{N}}.
\]

For the test vector entry,
\[
K_N(\mathbb{P}_i,\mathbb{P})-K(\mathbb{P}_i,\mathbb{P})
=
\langle \widehat\mu_i-\mu_{\mathbb{P}_i},\mu_{\mathbb{P}}\rangle_{\mathcal{H}_k}.
\]
Thus,
\[
|K_N(\mathbb{P}_i,\mathbb{P})-K(\mathbb{P}_i,\mathbb{P})|
\le
\|\widehat\mu_i-\mu_{\mathbb{P}_i}\|_{\mathcal{H}_k}
\|\mu_{\mathbb{P}}\|_{\mathcal{H}_k}
\le
C_\mu\kappa^2
\sqrt{\frac{\log(2n/\delta)}{N}}.
\]
\end{proof}

\begin{lemma}
\label{lem:rf_bag_empirical_app}
On the event of Theorem~\ref{thm:uniform_rf_app}, for all $i,j\in\{1,\ldots,n\}$,
\begin{equation}
\label{eq:rf_bag_empirical_app}
    |K_{D,N}(\mathbb{P}_i,\mathbb{P}_j)-K_N(\mathbb{P}_i,\mathbb{P}_j)|
    \le
    \epsilon_D(\delta),
\end{equation}
and for any test distribution $\mathbb{P}$,
\begin{equation}
\label{eq:rf_test_entry_app}
    |K_{D,N}(\mathbb{P}_i,\mathbb{P})-K_N(\mathbb{P}_i,\mathbb{P})|
    \le
    \epsilon_D(\delta).
\end{equation}
\end{lemma}

\begin{proof}
We only need to prove \eqref{eq:rf_bag_empirical_app} as the proof of \eqref{eq:rf_test_entry_app} is identical. Since the empirical bags define probability measures, we can write
\[
K_{D,N}(\mathbb{P}_i,\mathbb{P}_j)-K_N(\mathbb{P}_i,\mathbb{P}_j)
=
\int_{\mathcal{X}}\int_{\mathcal{X}}
\{k_D(\xb,\zb)-k(\xb,\zb)\}
\,d\widehat{\mathbb{P}}_i(\xb)d\widehat{\mathbb{P}}_j(\zb).
\]
Therefore,
\[
|K_{D,N}(\mathbb{P}_i,\mathbb{P}_j)-K_N(\mathbb{P}_i,\mathbb{P}_j)|
\le
\sup_{\xb,\zb\in\mathcal{X}}
|k_D(\xb,\zb)-k(\xb,\zb)|
\le
\epsilon_D(\delta).
\]
\end{proof}

\subsection{Matrix perturbation bounds}
\label{app:RaDR_matrix_perturbation}

Define
\[
a_N(\delta)
=
2C_\mu\kappa^2
\sqrt{\frac{\log(2n/\delta)}{N}},
\qquad
b_N(\delta)
=
C_\mu\kappa^2
\sqrt{\frac{\log(2n/\delta)}{N}}.
\]
Recall $\Kb$ denote the population distribution Gram matrix, $\Kb_N$ the exact finite-bag Gram matrix, and $\widehat{\Kb}_{D,N}$ the RF finite-bag Gram matrix.

The corresponding test vectors have components
\[
\begin{aligned}
[\vb(\mathbb P)]_i&=K(\mathbb P_i,\mathbb P),\\
[\vb_N(\mathbb P)]_i&=K_N(\mathbb P_i,\mathbb P),\\
[\widehat{\vb}_{D,N}(\mathbb P)]_i&=K_{D,N}(\mathbb P_i,\mathbb P).
\end{aligned}
\]
Thus the three predictors evaluated on a true test distribution are
\[
\begin{aligned}
\widehat f_\lambda(\mathbb P)
&=\vb(\mathbb P)^\top(\Kb+n\lambda\Ib)^{-1}\yb,\\
\widehat f_{\lambda,N}(\mathbb P)
&=\vb_N(\mathbb P)^\top(\Kb_N+n\lambda\Ib)^{-1}\yb,\\
\widehat f_{\lambda,N,D}(\mathbb P)
&=\widehat{\vb}_{D,N}(\mathbb P)^\top
(\widehat{\Kb}_{D,N}+n\lambda\Ib)^{-1}\yb.
\end{aligned}
\]

\begin{lemma}
\label{lem:gram_perturb_exact_app}
Assume that the event of Lemma~\ref{lem:two_stage_bag_kernel_app} holds. Then
\begin{equation}
\label{eq:gram_diff_bound_exact_app}
    \|\Kb_N-\Kb\|_2
    \le
    n a_N(\delta),
\end{equation}
and for every test distribution $\mathbb{P}$,
\begin{equation}
\label{eq:test_vec_diff_bound_exact_app}
    \|\vb_N(\mathbb{P})-\vb(\mathbb{P})\|_2
    \le
    \sqrt{n}\,b_N(\delta).
\end{equation}
\end{lemma}

\begin{proof}
Each entry of $\Kb_N-\Kb$ is bounded in absolute value by $a_N(\delta)$. Hence,
\[
\|\Kb_N-\Kb\|_2
\le
\|\Kb_N-\Kb\|_F
\le
n a_N(\delta).
\]
Similarly, each entry of $\vb_N(\mathbb{P})-\vb(\mathbb{P})$ is bounded by $b_N(\delta)$ in absolute value, so
\[
\|\vb_N(\mathbb{P})-\vb(\mathbb{P})\|_2
\le
\sqrt n\,b_N(\delta).
\]
\end{proof}

\begin{lemma}
\label{lem:gram_perturb_rf_app}
Assume that the event of Lemma~\ref{lem:rf_bag_empirical_app} holds. Then
\begin{equation}
\label{eq:gram_diff_bound_rf_app}
    \|\widehat{\Kb}_{D,N}-\Kb_N\|_2
    \le
    n\epsilon_D(\delta),
\end{equation}
and for every test distribution $\mathbb{P}$,
\begin{equation}
\label{eq:test_vec_diff_bound_rf_app}
    \|\widehat{\vb}_{D,N}(\mathbb{P})-\vb_N(\mathbb{P})\|_2
    \le
    \sqrt n\,\epsilon_D(\delta).
\end{equation}
\end{lemma}

\begin{proof}
Each entry of $\widehat{\Kb}_{D,N}-\Kb_N$ is bounded in absolute value by $\epsilon_D(\delta)$. Hence,
\[
\|\widehat{\Kb}_{D,N}-\Kb_N\|_2
\le
\|\widehat{\Kb}_{D,N}-\Kb_N\|_F
\le
n\epsilon_D(\delta).
\]
The test-vector bound follows in the same way because each component of $\widehat{\vb}_{D,N}(\mathbb{P})-\vb_N(\mathbb{P})$ is bounded by $\epsilon_D(\delta)$.
\end{proof}

\subsection{Predictor perturbation}
\label{app:RaDR_predictor_perturbation}

\begin{lemma}
\label{lem:resolvent_app}
For any symmetric matrices $\Ab,\Bb\in\mathbb{R}^{n\times n}$ and $\lambda>0$ for which both shifted matrices are invertible,
\begin{equation}
\label{eq:resolvent_app}
    (\Ab+n\lambda \Ib)^{-1}-(\Bb+n\lambda \Ib)^{-1}
    =
    (\Ab+n\lambda \Ib)^{-1}
    (\Bb-\Ab)
    (\Bb+n\lambda \Ib)^{-1}.
\end{equation}
Moreover, if $\Ab$ and $\Bb$ are positive semidefinite, then
\[
\|(\Ab+n\lambda\Ib)^{-1}\|_2
\le
\frac{1}{n\lambda},
\qquad
\|(\Bb+n\lambda\Ib)^{-1}\|_2
\le
\frac{1}{n\lambda}.
\]
\end{lemma}

\begin{proof}
The identity follows by direct multiplication. The inverse norm bound follows because all eigenvalues of $\Ab+n\lambda\Ib$ and $\Bb+n\lambda\Ib$ are at least $n\lambda$ when $\Ab$ and $\Bb$ are positive semidefinite.
\end{proof}

\begin{lemma}
\label{lem:predictor_exact_vs_twostage_app}
Assume the event of Lemma~\ref{lem:gram_perturb_exact_app} holds. Then there exists a constant $C_{K,y}>0$, depending only on $\kappa$ and $C_y$, such that for every test distribution $\mathbb{P}$,
\begin{equation}
\label{eq:predictor_exact_vs_twostage_app}
    |\widehat f_{\lambda,N}(\mathbb{P})-\widehat f_\lambda(\mathbb{P})|
    \le
    C_{K,y}
    \left(
    \frac{b_N(\delta)}{\lambda}
    +
    \frac{a_N(\delta)}{\lambda^2}
    \right).
\end{equation}
\end{lemma}

\begin{proof}
Set
\[
\Ab_N=\Kb_N+n\lambda\Ib,
\qquad
\Ab=\Kb+n\lambda\Ib.
\]
Subtracting the one-stage estimator from the exact two-stage estimator gives
\[
\widehat f_{\lambda,N}(\mathbb{P})-\widehat f_\lambda(\mathbb{P})
=
(\vb_N(\mathbb{P})-\vb(\mathbb{P}))^\top \Ab_N^{-1}\yb
+
\vb(\mathbb{P})^\top(\Ab_N^{-1}-\Ab^{-1})\yb.
\]
Denote these two terms by $B_1$ and $B_2$.

For $B_1$, using Cauchy--Schwarz, Lemma~\ref{lem:resolvent_app}, and \eqref{eq:test_vec_diff_bound_exact_app},
\[
|B_1|
\le
\|\vb_N(\mathbb{P})-\vb(\mathbb{P})\|_2
\|\Ab_N^{-1}\|_2
\|\yb\|_2
\le
\frac{\sqrt n\,b_N(\delta)}{n\lambda}
\sqrt n C_y
=
\frac{C_y b_N(\delta)}{\lambda}.
\]

For $B_2$, the resolvent identity in Lemma~\ref{lem:resolvent_app} gives
\[
\Ab_N^{-1}-\Ab^{-1}
=
\Ab_N^{-1}(\Kb-\Kb_N)\Ab^{-1}.
\]
Therefore,
\[
|B_2|
\le
\|\vb(\mathbb{P})\|_2
\|\Ab_N^{-1}\|_2
\|\Kb_N-\Kb\|_2
\|\Ab^{-1}\|_2
\|\yb\|_2.
\]
Each component of $\vb(\mathbb{P})$ satisfies
\[
|K(\mathbb{P}_i,\mathbb{P})|\le \kappa^2,
\]
so $\|\vb(\mathbb{P})\|_2\le \sqrt n\,\kappa^2$. Using $\|\yb\|_2\le \sqrt n C_y$ and \eqref{eq:gram_diff_bound_exact_app},
\[
|B_2|
\le
\frac{\sqrt n\,\kappa^2\sqrt n C_y}{(n\lambda)^2}
n a_N(\delta)
=
\frac{\kappa^2C_y a_N(\delta)}{\lambda^2}.
\]
Combining the two bounds and absorbing constants into $C_{K,y}$ proves the claim.
\end{proof}

\begin{lemma}
\label{lem:predictor_rf_vs_exact2_app}
Assume the event of Lemma~\ref{lem:gram_perturb_rf_app} holds. Then there exists a constant $C_{K,y}>0$, depending only on $\kappa$ and $C_y$, such that for every test distribution $\mathbb{P}$,
\begin{equation}
\label{eq:predictor_rf_vs_exact2_app}
    |\widehat f_{\lambda,N,D}(\mathbb{P})-\widehat f_{\lambda,N}(\mathbb{P})|
    \le
    C_{K,y}
    \left(
    \frac{\epsilon_D(\delta)}{\lambda}
    +
    \frac{\epsilon_D(\delta)}{\lambda^2}
    \right).
\end{equation}
\end{lemma}

\begin{proof}
Set
\[
\Ab_{D,N}=\widehat{\Kb}_{D,N}+n\lambda\Ib,
\qquad
\Ab_N=\Kb_N+n\lambda\Ib.
\]
Subtracting the exact two-stage predictor from the RF two-stage predictor gives
\[
\widehat f_{\lambda,N,D}(\mathbb{P})-\widehat f_{\lambda,N}(\mathbb{P})
=
(\widehat{\vb}_{D,N}(\mathbb{P})-\vb_N(\mathbb{P}))^\top
\Ab_{D,N}^{-1}\yb
+
\vb_N(\mathbb{P})^\top
(\Ab_{D,N}^{-1}-\Ab_N^{-1})
\yb.
\]
Denote these two terms by $C_1$ and $C_2$.

For $C_1$,
\[
|C_1|
\le
\|\widehat{\vb}_{D,N}(\mathbb{P})-\vb_N(\mathbb{P})\|_2
\|\Ab_{D,N}^{-1}\|_2
\|\yb\|_2
\le
\frac{\sqrt n\,\epsilon_D(\delta)}{n\lambda}
\sqrt n C_y
=
\frac{C_y\epsilon_D(\delta)}{\lambda}.
\]
For $C_2$, the resolvent identity gives
\[
\Ab_{D,N}^{-1}-\Ab_N^{-1}
=
\Ab_{D,N}^{-1}
(\Kb_N-\widehat{\Kb}_{D,N})
\Ab_N^{-1}.
\]
Thus,
\[
|C_2|
\le
\|\vb_N(\mathbb{P})\|_2
\|\Ab_{D,N}^{-1}\|_2
\|\widehat{\Kb}_{D,N}-\Kb_N\|_2
\|\Ab_N^{-1}\|_2
\|\yb\|_2.
\]
Since $|K_N(\mathbb{P}_i,\mathbb{P})|\le \kappa^2$, we have $\|\vb_N(\mathbb{P})\|_2\le \sqrt n\,\kappa^2$. Combining this with $\|\yb\|_2\le \sqrt n C_y$ and \eqref{eq:gram_diff_bound_rf_app},
\[
|C_2|
\le
\frac{\sqrt n\,\kappa^2\sqrt n C_y}{(n\lambda)^2}
n\epsilon_D(\delta)
=
\frac{\kappa^2C_y\epsilon_D(\delta)}{\lambda^2}.
\]
Combining the two bounds proves the claim.
\end{proof}

\begin{lemma}
\label{lem:excess_transfer_app}
For squared loss,
\begin{equation}
\label{eq:excess_identity_app}
    \mathcal{E}(f,f_\rho)
    =
    \|f-f_\rho\|_{L^2(\rho_X)}^2.
\end{equation}
Consequently, for any predictors $f$ and $g$,
\begin{equation}
\label{eq:excess_transfer_l2_app}
    \mathcal{E}(f,f_\rho)
    \le
    2\mathcal{E}(g,f_\rho)
    +
    2\|f-g\|_{L^2(\rho_X)}^2.
\end{equation}
If $\sup_{\mathbb{P}}|f(\mathbb{P})-g(\mathbb{P})|\le \Delta$, then
\begin{equation}
\label{eq:excess_transfer_sup_app}
    \mathcal{E}(f,f_\rho)
    \le
    2\mathcal{E}(g,f_\rho)+2\Delta^2.
\end{equation}
\end{lemma}

\begin{proof}
The identity \eqref{eq:excess_identity_app} is the standard orthogonality identity for squared loss:
\[
\mathcal{R}[f]-\mathcal{R}[f_\rho]
=
\mathbb{E}_{\mathbb{P}\sim\rho_X}
\left[
(f(\mathbb{P})-f_\rho(\mathbb{P}))^2
\right].
\]
Now write
\[
f-f_\rho=(f-g)+(g-f_\rho).
\]
Using $(u+v)^2\le 2u^2+2v^2$ and integrating with respect to $\rho_X$ gives \eqref{eq:excess_transfer_l2_app}. The last claim follows from
\[
\|f-g\|_{L^2(\rho_X)}^2
\le
\sup_{\mathbb{P}}|f(\mathbb{P})-g(\mathbb{P})|^2.
\]
\end{proof}

\subsection{Proofs of the exact two-stage and RF risk bounds}
\label{app:proof_RaDR_main}

For $\eta\in(0,1)$, define the deterministic perturbation error
\[
\begin{aligned}
\Delta_N(\eta)
&=C_{K,y}\left(
\frac{b_N(\eta)}{\lambda}
+\frac{a_N(\eta)}{\lambda^2}\right),\\
\Delta_D(\eta)
&=C_{K,y}\epsilon_D(\eta)
\left(\lambda^{-1}+\lambda^{-2}\right).
\end{aligned}
\]
Lemmas~\ref{lem:predictor_exact_vs_twostage_app}
and~\ref{lem:predictor_rf_vs_exact2_app} imply the corresponding
uniform predictor bounds on events of probability at least
$1-\eta$. Moreover,
\[
\Delta_N(\eta)^2
\le C\frac{\log(2n/\eta)}{N}
\left(\lambda^{-2}+\lambda^{-4}\right),
\]
and
\[
\Delta_D(\eta)^2
\le C\epsilon_D(\eta)^2
\left(\lambda^{-2}+\lambda^{-4}\right),
\]
where $C$ depends only on the fixed boundedness constants.

\begin{proof}[Proof of Theorem~\ref{thm:exact_twostage_main}]
Use failure probability $\delta/2$ for the one-stage baseline
and $\delta/2$ for the training-bag concentration event.
Their intersection has probability at least $1-\delta$.
On this intersection, Lemma~\ref{lem:excess_transfer_app} gives
\[
\mathcal E(\widehat f_{\lambda,N},f_\rho)
\le 2B_{\mathrm{1s}}(n,\lambda,\delta/2)
+2\Delta_N(\delta/2)^2.
\]
The concentration term involves $\log(4n/\delta)$.
Since $n\ge1$ and $\delta\in(0,1)$,
\[
\log(4n/\delta)\le2\log(2n/\delta).
\]
Absorbing this numerical factor into $C_2$ proves
\eqref{eq:exact_twostage_main_sec3}.
\end{proof}

\begin{proof}[Proof of Theorem~\ref{thm:RaDR_excess_main}]
Allocate failure probability $\delta/3$ separately to:
the one-stage baseline, the training-bag concentration event,
and the uniform RF approximation event.
A union bound gives simultaneous validity with probability
at least $1-\delta$; independence of these events is not needed.

On their intersection, applying
Lemma~\ref{lem:excess_transfer_app} twice gives
\[
\begin{aligned}
\mathcal E(\widehat f_{\lambda,N,D},f_\rho)
&\le 2\mathcal E(\widehat f_{\lambda,N},f_\rho)
+2\Delta_D(\delta/3)^2\\
&\le 4\mathcal E(\widehat f_\lambda,f_\rho)
+4\Delta_N(\delta/3)^2+2\Delta_D(\delta/3)^2\\
&\le 4B_{\mathrm{1s}}(n,\lambda,\delta/3)
+4\Delta_N(\delta/3)^2+2\Delta_D(\delta/3)^2.
\end{aligned}
\]
Finally,
\[
\log(6n/\delta)\le2\log(3n/\delta).
\]
Substituting the perturbation-radius bounds and absorbing
fixed numerical factors into $C_3$ proves
\eqref{eq:RaDR_excess_main}.
\end{proof}

\subsection{Population excess risk versus deployed empirical test bags}
\label{app:RaDR_deployment}

Theorems~\ref{thm:exact_twostage_main} and~\ref{thm:RaDR_excess_main} are population statements computed on true test distributions $\mathbb{P}\sim\rho_X$. We now analyze the additional perturbation that appears when, at deployment time, a new test distribution is available only through a fresh bag. Conditional on the test distribution, this bag is independent of the training data, the random features, and the test response.

Let
\[
\widetilde{\xb}_{1},\ldots,\widetilde{\xb}_{N_t}
\overset{\mathrm{i.i.d.}}{\sim}
\mathbb{P},
\qquad
\widehat{\mathbb{P}}^{N_t}
=
\frac{1}{N_t}\sum_{r=1}^{N_t}\delta_{\widetilde{\xb}_r}.
\]
Define the deployed predictors
\[
\widehat f_{\lambda,N}^{(N_t)}(\mathbb{P})
=
\widehat f_{\lambda,N}(\widehat{\mathbb{P}}^{N_t}),
\qquad
\widehat f_{\lambda,N,D}^{(N_t)}(\mathbb{P})
=
\widehat f_{\lambda,N,D}(\widehat{\mathbb{P}}^{N_t}).
\]

\begin{lemma}
\label{lem:fresh_test_embedding_app}
Under Assumption~\ref{ass:RaDR_bounded}, for any fixed test distribution $\mathbb{P}$ and any $\delta\in(0,1)$,
\[
\mathbb{P}\left(
\|\widehat\mu_{\mathbb{P}}^{N_t}-\mu_{\mathbb{P}}\|_{\mathcal{H}_k}
\le
C_\mu\kappa
\sqrt{\frac{\log(2/\delta)}{N_t}}
\right)
\ge
1-\delta.
\]
\end{lemma}

\begin{proof}
Condition on the test distribution $\mathbb{P}$. Then
\[
\widehat\mu_{\mathbb{P}}^{N_t}-\mu_{\mathbb{P}}
=
\frac{1}{N_t}\sum_{r=1}^{N_t}
\left(
k(\cdot,\widetilde{\xb}_r)-\mu_{\mathbb{P}}
\right),
\qquad
\widetilde{\xb}_r\overset{\mathrm{i.i.d.}}{\sim}\mathbb{P}.
\]
The summands are centered Hilbert-space-valued random variables with norm at most $2\kappa$. Apply Lemma~\ref{lem:mean_concentration_app} with sample size $N_t$ to obtain the result.
\end{proof}

\begin{proposition}
\label{prop:deployment_exact_app}
Under Assumption~\ref{ass:RaDR_bounded}, fix the training sample and a test distribution $\mathbb{P}$. Then, with probability at least $1-\delta$ over the test bag, we have
\begin{equation}
\label{eq:deployment_exact_bound_app}
    |\widehat f_{\lambda,N}^{(N_t)}(\mathbb{P})-\widehat f_{\lambda,N}(\mathbb{P})|
    \le
    \frac{C_\mu C_y\kappa^2}{\lambda}
    \sqrt{\frac{\log(2/\delta)}{N_t}}.
\end{equation}
\end{proposition}

\begin{proof}
By definition,
\[
\widehat f_{\lambda,N}^{(N_t)}(\mathbb{P})-\widehat f_{\lambda,N}(\mathbb{P})
=
\left(
\vb_N(\widehat{\mathbb{P}}^{N_t})-\vb_N(\mathbb{P})
\right)^\top
(\Kb_N+n\lambda\Ib)^{-1}\yb.
\]
Thus,
\[
|\widehat f_{\lambda,N}^{(N_t)}(\mathbb{P})-\widehat f_{\lambda,N}(\mathbb{P})|
\le
\|\vb_N(\widehat{\mathbb{P}}^{N_t})-\vb_N(\mathbb{P})\|_2
\|(\Kb_N+n\lambda\Ib)^{-1}\|_2
\|\yb\|_2.
\]
We have
\[
\|(\Kb_N+n\lambda\Ib)^{-1}\|_2\le \frac{1}{n\lambda},
\qquad
\|\yb\|_2\le \sqrt n C_y.
\]
For each $i$,
\[
K_N(\mathbb{P}_i,\widehat{\mathbb{P}}^{N_t})
-
K_N(\mathbb{P}_i,\mathbb{P})
=
\langle \widehat\mu_i,\widehat\mu_{\mathbb{P}}^{N_t}-\mu_{\mathbb{P}}\rangle_{\mathcal{H}_k}.
\]
Since $\|\widehat\mu_i\|_{\mathcal{H}_k}\le \kappa$, Lemma~\ref{lem:fresh_test_embedding_app} implies
\[
|K_N(\mathbb{P}_i,\widehat{\mathbb{P}}^{N_t})
-
K_N(\mathbb{P}_i,\mathbb{P})|
\le
C_\mu\kappa^2
\sqrt{\frac{\log(2/\delta)}{N_t}}.
\]
Therefore,
\[
\|\vb_N(\widehat{\mathbb{P}}^{N_t})-\vb_N(\mathbb{P})\|_2
\le
\sqrt n\,C_\mu\kappa^2
\sqrt{\frac{\log(2/\delta)}{N_t}}.
\]
Substitution gives \eqref{eq:deployment_exact_bound_app}.
\end{proof}

\begin{lemma}
\label{lem:fresh_test_rf_embedding_app}
For any fixed test distribution $\mathbb{P}$, any set of random features $(\omegab_r,b_r)_{r=1}^D$, and any $\delta\in(0,1)$, there exists a universal constant $C_\phi>0$ such that
\[
\mathbb{P}\left(
\|\widehat{\mub}_{\mathbb{P},D}^{N_t}-\mub_{\mathbb{P},D}\|_2
\le
C_\phi
\sqrt{\frac{\log(2/\delta)}{N_t}}
\right)
\ge
1-\delta.
\]
\end{lemma}

\begin{proof}
Condition on $\mathbb{P}$ and on the realized random features. The random-bias cosine map satisfies
\[
\|\phib_D(\xb)\|_2^2
=
\frac{2}{D}\sum_{r=1}^D
\cos^2(\omegab_r^\top\xb+b_r)
\le
2.
\]
Thus,
\[
\|\phib_D(\widetilde{\xb}_r)-\mub_{\mathbb{P},D}\|_2
\le
2\sqrt 2.
\]
The expectation and bounded-differences argument in Lemma~\ref{lem:mean_concentration_app}, with feature-norm bound $\sqrt2$ and sample size $N_t$, gives the result. One may take $C_\phi=\sqrt2 C_\mu$.
\end{proof}

\begin{proposition}
\label{prop:deployment_rf_app}
Under Assumption~\ref{ass:RaDR_bounded}, fix the training sample, the realized random features, and a test distribution $\mathbb{P}$. Then, with probability at least $1-\delta$ over the test bag,
\begin{equation}
\label{eq:deployment_rf_bound_app}
    |\widehat f_{\lambda,N,D}^{(N_t)}(\mathbb{P})-\widehat f_{\lambda,N,D}(\mathbb{P})|
    \le
    \frac{\sqrt2 C_\phi C_y}{\lambda}
    \sqrt{\frac{\log(2/\delta)}{N_t}}.
\end{equation}
\end{proposition}

\begin{proof}
By definition,
\[
\widehat f_{\lambda,N,D}^{(N_t)}(\mathbb{P})-\widehat f_{\lambda,N,D}(\mathbb{P})
=
\left(
\widehat{\vb}_{D,N}(\widehat{\mathbb{P}}^{N_t})
-
\widehat{\vb}_{D,N}(\mathbb{P})
\right)^\top
(\widehat{\Kb}_{D,N}+n\lambda\Ib)^{-1}\yb.
\]
Hence,
\[
|\widehat f_{\lambda,N,D}^{(N_t)}(\mathbb{P})-\widehat f_{\lambda,N,D}(\mathbb{P})|
\le
\|\widehat{\vb}_{D,N}(\widehat{\mathbb{P}}^{N_t})-\widehat{\vb}_{D,N}(\mathbb{P})\|_2
\frac{1}{n\lambda}
\sqrt n C_y.
\]
For each $i$,
\[
K_{D,N}(\mathbb{P}_i,\widehat{\mathbb{P}}^{N_t})
-
K_{D,N}(\mathbb{P}_i,\mathbb{P})
=
\langle \widehat{\mub}_{i,D},
\widehat{\mub}_{\mathbb{P},D}^{N_t}-\mub_{\mathbb{P},D}
\rangle_2.
\]
Since $\|\widehat{\mub}_{i,D}\|_2\le \sqrt 2$, Lemma~\ref{lem:fresh_test_rf_embedding_app} implies
\[
|K_{D,N}(\mathbb{P}_i,\widehat{\mathbb{P}}^{N_t})
-
K_{D,N}(\mathbb{P}_i,\mathbb{P})|
\le
\sqrt 2 C_\phi
\sqrt{\frac{\log(2/\delta)}{N_t}}.
\]
Therefore,
\[
\|\widehat{\vb}_{D,N}(\widehat{\mathbb{P}}^{N_t})-\widehat{\vb}_{D,N}(\mathbb{P})\|_2
\le
\sqrt{2n}C_\phi
\sqrt{\frac{\log(2/\delta)}{N_t}}.
\]
Substitution proves \eqref{eq:deployment_rf_bound_app}.
\end{proof}

\begin{remark}
The population excess-risk theorems measure generalization on true test distributions. The deployment perturbation bounds measure the additional pointwise error caused by replacing a true test distribution by a fresh empirical test bag. The latter contributes a separate $N_t^{-1/2}$ term with only one factor of $\lambda^{-1}$.
\end{remark}

\subsection{Proof of the no-loss result (Corollary~\ref{cor:RaDR_D_threshold_app})}
\label{app:RaDR_noloss_proofs}

\begin{proof}
Work under the assumptions of Theorem~\ref{thm:RaDR_excess_main}
and fix $\delta\in(0,1)$. Let $N=N_n\to\infty$,
$N_n\le n^p$, and $D_n\asymp N_n$.
Theorem~\ref{thm:uniform_rf_app} gives, for sufficiently large $n$,
\[
\epsilon_{D_n}(\delta/3)^2
\le C_\delta\frac{\log(eD_n)}{D_n}.
\]
Because $D_n\asymp N_n$ and $N_n\le n^p$, there are constants
$c_1,c_2>0$ such that, for all sufficiently large $n$,
\[
D_n\ge c_1N_n,
\qquad
\log(eD_n)\le c_2\log n.
\]
Consequently,
\[
\epsilon_{D_n}(\delta/3)^2
\le C\frac{\log(3n/\delta)}{N_n}.
\]

Define
\[
G_n=\frac{\log(3n/\delta)}{N_n}
\left(\lambda_n^{-2}+\lambda_n^{-4}\right).
\]
Theorem~\ref{thm:RaDR_excess_main} therefore implies
\[
\mathcal E(\widehat f_{\lambda_n,N_n,D_n},f_\rho)
\le 4B_{\mathrm{1s}}(n,\lambda_n,\delta/3)+CG_n
\]
with probability at least $1-\delta$.
Here $\lambda_n$ must satisfy the chosen baseline conditions
at confidence level $\delta/3$.

For comparison, applying
Theorem~\ref{thm:exact_twostage_main} with failure probability
$2\delta/3$ yields
\[
\mathcal E(\widehat f_{\lambda_n,N_n},f_\rho)
\le 2B_{\mathrm{1s}}(n,\lambda_n,\delta/3)+C_2G_n
\]
with probability at least $1-2\delta/3$.
Thus the two guarantees have the same upper-bound order under
a common probability. Both estimates can also be
controlled simultaneously on the three-event intersection
used in the proof of Theorem~\ref{thm:RaDR_excess_main}.

For computation, use direct feature construction and dense
linear algebra, with fixed instance dimension and fixed kernel
parameters. Choosing the less expensive ridge formulation gives
\[
T_{\mathrm{RaDR}}
=O\!\left(nND+\min\{nD^2+D^3,\;n^2D+n^3\}\right).
\]
When $N_n\le n$, the primal formulation and $D_n\asymp N_n$
give
\[
T_{\mathrm{RaDR}}
\le O(nN_n^2+N_n^3)=O(nN_n^2).
\]
When $N_n\ge n$, the dual formulation gives
\[
T_{\mathrm{RaDR}}
\le O(nN_n^2+n^2N_n+n^3)=O(nN_n^2).
\]
The conventional dense exact-kernel implementation has cost
\[
T_{\mathrm{exact}}^{\mathrm{dense}}
=\Theta(n^2N_n^2+n^3).
\]
Therefore,
\[
\frac{T_{\mathrm{RaDR}}}{T_{\mathrm{exact}}^{\mathrm{dense}}}
=O\!\left(\frac{nN_n^2}{n^2N_n^2+n^3}\right)
=O(n^{-1}).
\]
\end{proof}

\subsection{SVM perturbation details}
\label{app:RaDR_svm}

The two-stage analysis above is specific to kernel ridge regression with squared loss. For SVM, the loss is the hinge loss. Here, we derive self-contained objective-perturbation results for two-stage distribution classification and its random-feature approximation.

Throughout this subsection, assume binary labels $y_i\in\{-1,+1\}$ and define
\[
\ell_{\mathrm h}(u)=(1-u)_+.
\]
The hinge loss is convex and $1$-Lipschitz:
\[
|\ell_{\mathrm h}(u)-\ell_{\mathrm h}(v)|\le |u-v|.
\]

For $\wb\in\mathcal{H}_k$, define the exact empirical SVM objective based on the latent bag embeddings by
\[
\mathcal{J}_n(\wb)
=
\frac1n\sum_{i=1}^n
\ell_{\mathrm h}\!\left(
y_i\langle \wb,\mu_{\mathbb{P}_i}\rangle_{\mathcal{H}_k}
\right)
+
\frac{\lambda}{2}\|\wb\|_{\mathcal{H}_k}^2.
\]
The observable exact two-stage objective is
\[
\mathcal{J}_{n,N}(\wb)
=
\frac1n\sum_{i=1}^n
\ell_{\mathrm h}\!\left(
y_i\langle \wb,\widehat\mu_i\rangle_{\mathcal{H}_k}
\right)
+
\frac{\lambda}{2}\|\wb\|_{\mathcal{H}_k}^2.
\]
For the RF classifier, define
\[
\mathcal{J}_{n,D}(\betab)
=
\frac1n\sum_{i=1}^n
\ell_{\mathrm h}\!\left(
y_i\betab^\top\mub_{\mathbb{P}_i,D}
\right)
+
\frac{\lambda}{2}\|\betab\|_2^2,
\]
and
\[
\mathcal{J}_{n,N,D}(\betab)
=
\frac1n\sum_{i=1}^n
\ell_{\mathrm h}\!\left(
y_i\betab^\top\widehat{\mub}_{i,D}
\right)
+
\frac{\lambda}{2}\|\betab\|_2^2.
\]

\begin{lemma}
\label{lem:svm_norm_bounds_app}
Any minimizer of any one of the four objectives above satisfies the corresponding norm bound
\[
\|\wb\|_{\mathcal{H}_k}\le \sqrt{\frac2\lambda},
\qquad
\text{or}
\qquad
\|\betab\|_2\le \sqrt{\frac2\lambda}.
\]
\end{lemma}

\begin{proof}
We prove the RKHS case as the Euclidean case is identical. Evaluating the objective at $\wb=0$ gives value $1$. By optimality of the minimizer $\widehat{\wb}$,
\[
\mathcal{J}(\widehat{\wb})\le 1.
\]
Since the hinge-loss term is nonnegative,
\[
\frac{\lambda}{2}\|\widehat{\wb}\|_{\mathcal{H}_k}^2\le 1.
\]
Therefore $\|\widehat{\wb}\|_{\mathcal{H}_k}\le \sqrt{2/\lambda}$.
\end{proof}

\begin{lemma}
\label{lem:svm_primal_obj_perturb_app}
Assume the event of Lemma~\ref{lem:mean_concentration_app}. Define
\[
\Delta_N(\delta)
=
C_\mu\kappa
\sqrt{\frac{\log(2n/\delta)}{N}}.
\]
Then for every $\wb\in\mathcal{H}_k$,
\[
|\mathcal{J}_{n,N}(\wb)-\mathcal{J}_n(\wb)|
\le
\|\wb\|_{\mathcal{H}_k}\Delta_N(\delta).
\]
Similarly, on the corresponding bounded-vector concentration event for empirical RF means, for every $\betab\in\mathbb{R}^D$,
\[
|\mathcal{J}_{n,N,D}(\betab)-\mathcal{J}_{n,D}(\betab)|
\le
\|\betab\|_2\Delta_N^{\mathrm{rf}}(\delta),
\]
where
\[
\Delta_N^{\mathrm{rf}}(\delta)
=
C_{\mu,\mathrm{rf}}
\sqrt{\frac{\log(2n/\delta)}{N}}.
\]
\end{lemma}

\begin{proof}
Since $\ell_{\mathrm h}$ is $1$-Lipschitz,
\[
\begin{aligned}
|\mathcal{J}_{n,N}(\wb)-\mathcal{J}_n(\wb)|
&\le
\frac1n\sum_{i=1}^n
\left|
\ell_{\mathrm h}\!\left(y_i\langle\wb,\widehat\mu_i\rangle_{\mathcal{H}_k}\right)
-
\ell_{\mathrm h}\!\left(y_i\langle\wb,\mu_{\mathbb{P}_i}\rangle_{\mathcal{H}_k}\right)
\right| \\
&\le
\frac1n\sum_{i=1}^n
|\langle\wb,\widehat\mu_i-\mu_{\mathbb{P}_i}\rangle_{\mathcal{H}_k}| \\
&\le
\|\wb\|_{\mathcal{H}_k}
\max_{1\le i\le n}
\|\widehat\mu_i-\mu_{\mathbb{P}_i}\|_{\mathcal{H}_k}.
\end{aligned}
\]
Applying Lemma~\ref{lem:mean_concentration_app} gives the claim. The RF statement follows in the same way, using $\|\phib_D(\xb)\|_2\le \sqrt2$.
\end{proof}

\begin{theorem}
\label{thm:svm_twostage_primal_app}
Let $\widehat{\wb}_{\lambda,n}$ and $\widehat{\wb}_{\lambda,n,N}$ be minimizers of $\mathcal{J}_n$ and $\mathcal{J}_{n,N}$, respectively. On the event of Lemma~\ref{lem:svm_primal_obj_perturb_app}, let
\[
B_\lambda=\sqrt{\frac2\lambda}.
\]
Then
\[
0
\le
\mathcal{J}_n(\widehat{\wb}_{\lambda,n,N})
-
\mathcal{J}_n(\widehat{\wb}_{\lambda,n})
\le
2B_\lambda\Delta_N(\delta).
\]
Moreover,
\[
\|\widehat{\wb}_{\lambda,n,N}-\widehat{\wb}_{\lambda,n}\|_{\mathcal{H}_k}
\le
2\sqrt{\frac{B_\lambda\Delta_N(\delta)}{\lambda}}.
\]
Consequently, for every test distribution $\mathbb{P}$,
\[
|\langle \widehat{\wb}_{\lambda,n,N}-\widehat{\wb}_{\lambda,n},
\mu_{\mathbb{P}}\rangle_{\mathcal{H}_k}|
\le
2\kappa
\sqrt{\frac{B_\lambda\Delta_N(\delta)}{\lambda}}.
\]
\end{theorem}

\begin{proof}
By Lemma~\ref{lem:svm_norm_bounds_app}, both minimizers lie in the norm ball of radius $B_\lambda$. Therefore Lemma~\ref{lem:svm_primal_obj_perturb_app} gives
\[
\mathcal{J}_n(\widehat{\wb}_{\lambda,n,N})
\le
\mathcal{J}_{n,N}(\widehat{\wb}_{\lambda,n,N})
+
B_\lambda\Delta_N(\delta).
\]
By optimality of $\widehat{\wb}_{\lambda,n,N}$ for $\mathcal{J}_{n,N}$,
\[
\mathcal{J}_{n,N}(\widehat{\wb}_{\lambda,n,N})
\le
\mathcal{J}_{n,N}(\widehat{\wb}_{\lambda,n}).
\]
Applying Lemma~\ref{lem:svm_primal_obj_perturb_app} again,
\[
\mathcal{J}_{n,N}(\widehat{\wb}_{\lambda,n})
\le
\mathcal{J}_n(\widehat{\wb}_{\lambda,n})
+
B_\lambda\Delta_N(\delta).
\]
Combining the three inequalities gives
\[
\mathcal{J}_n(\widehat{\wb}_{\lambda,n,N})
-
\mathcal{J}_n(\widehat{\wb}_{\lambda,n})
\le
2B_\lambda\Delta_N(\delta).
\]
The lower bound follows because $\widehat{\wb}_{\lambda,n}$ minimizes $\mathcal{J}_n$.

Next, $\mathcal{J}_n$ is $\lambda$-strongly convex because of the quadratic regularizer. Hence,
\[
\mathcal{J}_n(\widehat{\wb}_{\lambda,n,N})
\ge
\mathcal{J}_n(\widehat{\wb}_{\lambda,n})
+
\frac{\lambda}{2}
\|\widehat{\wb}_{\lambda,n,N}-\widehat{\wb}_{\lambda,n}\|_{\mathcal{H}_k}^2.
\]
Combining this with the objective gap bound gives the parameter perturbation bound. The score perturbation follows by Cauchy--Schwarz and $\|\mu_{\mathbb{P}}\|_{\mathcal{H}_k}\le\kappa$.
\end{proof}

\begin{lemma}
\label{lem:svm_dual_uniform_app}
Assume the event of Lemma~\ref{lem:rf_bag_empirical_app}. For
\[
\mathcal{A}
=
\{\alphab\in\mathbb{R}^n:0\le \alpha_i\le 1/n\},
\]
define
\[
\mathcal{D}_N(\alphab)
=
\sum_{i=1}^n\alpha_i
-
\frac{1}{2\lambda}
\alphab^\top\Yb\Kb_N\Yb\alphab,
\]
and
\[
\mathcal{D}_{N,D}(\alphab)
=
\sum_{i=1}^n\alpha_i
-
\frac{1}{2\lambda}
\alphab^\top\Yb\widehat{\Kb}_{D,N}\Yb\alphab.
\]
Then for every $\alphab\in\mathcal{A}$,
\[
|\mathcal{D}_{N,D}(\alphab)-\mathcal{D}_N(\alphab)|
\le
\frac{\epsilon_D(\delta)}{2\lambda}.
\]
\end{lemma}

\begin{proof}
For feasible $\alphab$,
\[
\begin{aligned}
|\mathcal{D}_{N,D}(\alphab)-\mathcal{D}_N(\alphab)|
&=
\frac{1}{2\lambda}
\left|
\alphab^\top
\Yb(\widehat{\Kb}_{D,N}-\Kb_N)\Yb
\alphab
\right| \\
&\le
\frac{1}{2\lambda}
\sum_{i,j=1}^n
\alpha_i\alpha_j
|[\widehat{\Kb}_{D,N}-\Kb_N]_{ij}| \\
&\le
\frac{\epsilon_D(\delta)}{2\lambda}
\left(\sum_{i=1}^n\alpha_i\right)^2.
\end{aligned}
\]
Because $0\le\alpha_i\le 1/n$, we have $\sum_{i=1}^n\alpha_i\le1$. Therefore
\[
|\mathcal{D}_{N,D}(\alphab)-\mathcal{D}_N(\alphab)|
\le
\frac{\epsilon_D(\delta)}{2\lambda}.
\]
\end{proof}

We now proceed to prove Proposition~\ref{prop:svm_dual_sec3}.

\begin{proof}
The entrywise argument in Lemma~\ref{lem:svm_dual_uniform_app} is deterministic and applies directly with the approximation radius $\epsilon_D$ in the proposition, proving its first claim. Let $\alphab_N^\star$ and $\alphab_{N,D}^\star$ be maximizers of $\mathcal{D}_N$ and $\mathcal{D}_{N,D}$ over $\mathcal{A}$. Because $\alphab_N^\star$ maximizes $\mathcal{D}_N$,
\[
\mathcal{D}_N(\alphab_N^\star)
\ge
\mathcal{D}_N(\alphab_{N,D}^\star).
\]
Using Lemma~\ref{lem:svm_dual_uniform_app} at $\alphab_{N,D}^\star$ gives
\[
\mathcal{D}_N(\alphab_{N,D}^\star)
\ge
\mathcal{D}_{N,D}(\alphab_{N,D}^\star)
-
\frac{\epsilon_D}{2\lambda}.
\]
Hence,
\[
\mathcal{D}_N(\alphab_N^\star)
-
\mathcal{D}_{N,D}(\alphab_{N,D}^\star)
\ge
-
\frac{\epsilon_D}{2\lambda}.
\]
Repeating the same argument with the roles of $\mathcal{D}_N$ and $\mathcal{D}_{N,D}$ interchanged gives
\[
\mathcal{D}_N(\alphab_N^\star)
-
\mathcal{D}_{N,D}(\alphab_{N,D}^\star)
\le
\frac{\epsilon_D}{2\lambda}.
\]
Therefore the optimal dual objective values differ by at most $\epsilon_D/(2\lambda)$. By strong duality for soft-margin SVM with quadratic regularization, the same objective-value perturbation holds for the corresponding primal optima.
\end{proof}

\subsection{Proof of Proposition~\ref{prop:adaptive_trace_selection}}
\label{app:adaptive_trace_selection}

\begin{proof}
We first condition on the training data and a realized candidate
pool, so that the feature vectors and score matrix are fixed.
Substituting \eqref{eq:adaptive_selected_kernel} into the trace
objective gives
\[
\begin{aligned}
    \mathcal J_B(\widehat{\Kb}_S)
    &=
    \mathrm{tr}\!\left(
        \frac{1}{D}\sum_{m\in S}
        \widehat{\zb}_m\widehat{\zb}_m^\top\Bb
    \right)\\
    &=
    \frac{1}{D}\sum_{m\in S}
    \mathrm{tr}
    \big(\widehat{\zb}_m\widehat{\zb}_m^\top\Bb\big)\\
    &=
    \frac{1}{D}\sum_{m\in S}
    \widehat{\zb}_m^\top\Bb\widehat{\zb}_m
    =
    \frac{1}{D}\sum_{m\in S}q_{B,m}.
\end{aligned}
\]
This proves \eqref{eq:adaptive_selected_trace}.
Since $\Bb$ is held fixed across subsets, maximizing the
objective amounts to selecting the $D$ largest scores.
Indeed, if a selected score is smaller than an unselected one,
exchanging the two increases the objective. Repeating this
exchange produces a top-$D$ subset.

The first $D$ candidates form a feasible subset, and their
Gram matrix is
\[
    \widehat{\Kb}_{D,N}
    =
    \frac{1}{D}\sum_{m=1}^{D}
    \widehat{\zb}_m\widehat{\zb}_m^\top.
\]
The maximizing property of $S_B$ therefore gives
\[
    \mathcal J_B(\widehat{\Kb}_{S_B})
    \ge
    \mathcal J_B(\widehat{\Kb}_{D,N})
\]
for every realized pool, including when $\Bb$ is constructed
from that pool.

For the expectation statement, we further assume that $\Bb$
is fixed conditional on the training data and independent of
the candidate pool. To identify the expected Gram matrix,
consider any pair of bags $i$ and $j$. By
\eqref{eq:adaptive_empirical_scalar_feature},
\[
\begin{aligned}
    &\mathbb E_{\eta\sim\pi}
    [\widehat z_i(\eta)\widehat z_j(\eta)]\\
    &\quad=
    \frac{1}{N_iN_j}
    \sum_{a=1}^{N_i}\sum_{c=1}^{N_j}
    \mathbb E_{\omegab,b}
    \left[
        2\cos(\omegab^\top\xb_{ia}+b)
        \cos(\omegab^\top\xb_{jc}+b)
    \right]\\
    &\quad=
    \frac{1}{N_iN_j}
    \sum_{a=1}^{N_i}\sum_{c=1}^{N_j}
    k(\xb_{ia},\xb_{jc})\\
    &\quad=
    \langle\hat\mu_i,\hat\mu_j\rangle_{\mathcal H_k}
    =
    [\Kb_N]_{ij}.
\end{aligned}
\]
Here the second equality follows by averaging over the uniform
phase and then the spectral measure, while the final equality
uses the linear distribution kernel.
Since the candidates are sampled independently from $\pi$,
we obtain
\[
    \mathbb E_{\mathrm{pool}}
    [\widehat{\Kb}_{D,N}]
    =
    \Kb_N.
\]

The score matrix can now be taken outside the conditional
expectation, giving
\[
\begin{aligned}
    \mathbb E_{\mathrm{pool}}
    [\mathcal J_B(\widehat{\Kb}_{D,N})]
    &=
    \mathrm{tr}\!\left(
        \mathbb E_{\mathrm{pool}}
        [\widehat{\Kb}_{D,N}]\Bb
    \right)\\
    &=
    \mathrm{tr}(\Kb_N\Bb)
    =
    \mathcal J_B(\Kb_N).
\end{aligned}
\]
Finally, taking expectations of the realized-pool inequality
yields
\[
    \mathbb E_{\mathrm{pool}}
    [\mathcal J_B(\widehat{\Kb}_{S_B})]
    \ge
    \mathbb E_{\mathrm{pool}}
    [\mathcal J_B(\widehat{\Kb}_{D,N})]
    =
    \mathcal J_B(\Kb_N),
\]
which proves \eqref{eq:adaptive_expected_dominance}.
\end{proof}

%%%%%%%%%%%%%%%% REFERENCES %%%%%%%%%%%%%%%

\clearpage % Clear all remaining figures and tables then start a new page

\bibliography{bibliography} % for a file named science_template.bib
\bibliographystyle{plainnat}

\end{document}